\documentclass[journal, 10pt]{IEEEtran}
\usepackage{amsfonts}%
\usepackage{amssymb}%
\usepackage{setspace}
\usepackage[dvips]{graphicx}%
\usepackage{amsmath}%
\usepackage[usenames]{color}
\usepackage{algorithmic}
\usepackage{cite}
\usepackage{array}
\newtheorem{theorem}{Theorem}
\newtheorem{corollary}{Corollary}
\newtheorem{lemma}{Lemma}
\newtheorem{assumption}{Assumption}
\newtheorem{definition}{Definition}
\newtheorem{example-non*}{Example}

\allowdisplaybreaks 

\DeclareMathOperator*{\argmin}{arg\,min}
\DeclareMathOperator*{\argmax}{arg\,max}

\newcommand{\comment}[1]{}

\ifodd 1
\newcommand{\remove}[1]{}
\newcommand{\add}[1]{#1}
\else
\newcommand{\remove}[1]{#1}
\newcommand{\add}[1]{}
\fi

\ifodd 0
\newcommand{\bremove}[1]{}
\newcommand{\badd}[1]{#1}
\else
\newcommand{\bremove}[1]{#1}
\newcommand{\badd}[1]{}
\fi

\ifodd 1
\newcommand{\rmv}[1]{}
\else
\newcommand{\rmv}[1]{{\color{red}#1}}
\fi

\ifodd 0
\newcommand{\newc}[1]{{\color{blue}#1}} 
\else
\newcommand{\newc}[1]{#1}
\fi

\ifodd 0
\newcommand{\rev}[1]{{\color{blue}#1}} 
\newcommand{\cem}[1]{{\color{magenta}#1}} 
\newcommand{\com}[1]{\textbf{\color{red}(COMMENT: #1)}} 
\newcommand{\clar}[1]{\textbf{\color{green}(NEED CLARIFICATION: #1)}}

\else
\newcommand{\rev}[1]{#1}
\newcommand{\cem}[1]{#1}

\newcommand{\com}[1]{}
\newcommand{\clar}[1]{}
\fi

\begin{document}
\title{Decentralized Online Big Data Classification - a Bandit Framework}
\author{\IEEEauthorblockN{Cem Tekin*,~\IEEEmembership{Member,~IEEE}, Mihaela van der Schaar, ~\IEEEmembership{Fellow,~IEEE}\\}
\IEEEauthorblockA{Electrical Engineering Department,
University of California, Los Angeles\\
Email: cmtkn@ucla.edu, mihaela@ee.ucla.edu}
\thanks{A preliminary version of this work appeared in Allerton 2013. The work is partially supported by the grants NSF CNS 1016081 and  AFOSR DDDAS.}
}

\maketitle

\begin{abstract}
Distributed, online data mining systems have emerged as a result of applications requiring analysis of large amounts of correlated and high-dimensional data produced by multiple distributed data sources. 
\rmv{To correctly identify the events and phenomena of interest, different classification functions are required in order to process data exhibiting different \cem{features/}characteristics.}
\rmv{However, a learner may
not have access to the complete set of necessary classification functions to process a certain type of data. 
Moreover, the accuracy of each classification function for the incoming data stream is unknown a priori and
needs to be learned online.}
\cem{We propose a} distributed online data classification framework
where data is gathered by distributed data sources and processed by a heterogeneous set of distributed learners which
learn online, at run-time, 
how to classify the different data streams either by using their locally available classification functions or by helping each other by classifying each other's data. 
\rev{Importantly, since the data is gathered at different locations, sending the data to another learner to process incurs additional costs such as delays, and hence this will be only beneficial if the benefits obtained from a better classification will exceed the costs.}
We assume that the classification functions available to each processing element are fixed, but their prediction
accuracy for various types of incoming data are unknown and can change dynamically over time, and thus they need to be learned online.
%
We model \cem{the problem of joint classification by the
distributed and heterogeneous learners from multiple data sources}
as a cooperative contextual
bandit problem where each data 
is characterized by
a specific context. We develop distributed online learning algorithms for which we can prove that they have sublinear regret\rmv{, i.e. the
average error probability converges to the error probability of the best distributed classification scheme given the context information}.
\rmv{Our bounds hold uniformly over time, without requiring any assumptions about of the types of classification functions used.}
Compared to prior work in distributed online data mining, our work is the first to provide analytic regret results characterizing the performance of the proposed algorithms.
\comment{\cem{We also relate our distributed contextual learning approach to the notion of concept drift, which was introduced to address non-stationary learning problems over time, and we show that sublinear regret can be achieved when the concept drift is gradual without requiring a drift detection mechanism.}}
\rmv{Finally, we illustrate our proposed solutions using \cem{distributed} online data mining systems for network security and compare our results with existing state of the art solutions for online data mining.}
\end{abstract}


\add{\vspace{-0.2in}}
\section{Introduction}\label{sec:intro}

A plethora of Big Data applications (network security, surveillance, health monitoring, stock market prediction, intelligent traffic management, etc.) are emerging which
require online classification of large data sets collected from distributed network and traffic monitors, multimedia sources, sensor networks, etc.
This data is heterogeneous and dynamically evolves over time.
%
In this paper, we introduce a distributed online learning framework for classification of high-dimensional data
collected by distributed data sources. 

\cem{The distributedly collected data is processed by a set of decentralized heterogeneous learners equipped with classification functions with unknown accuracies.} \rev{In this setting communication, computation and sharing costs make it infeasible to use centralized data mining techniques where a single learner can access the entire data set.}
\rev{For example, in a wireless sensor surveillance network, nodes in different locations collect information about different events. 
The learners at each node of the network may run different classification algorithms, may have different resolution, processing speed, etc. 
\rmv{An event may happen rarely in one location while frequently in another location. Therefore, if the context implies that the event is a rare event, then it can be sent to a learner for which this event happens frequently to be classified.}
%
}

%
%
The input data stream and its associated context can be time-varying and heterogeneous. \cem{We use the term ``context'' generically, to represent}
any information related to the input data stream such as time, location and type \cem{(e.g., data features/characteristics/modality)} information.
Each learner can process (label) the incoming data in two different ways: either it can 
exploit its own information and its own classification functions or it can forward its input stream to another learner (possibly by incurring some cost) to have it labeled. A learner learns the accuracies of its own classification functions or other learners in an online way by comparing the result of the predictions with the true label of its input stream which is revealed at the end of each slot. The goal of each learner is to maximize its long term expected total reward, which is the expected number of correct labels minus the costs of classification. In this paper the cost is a generic term that can represent any known cost such as processing cost, delay cost, communication cost, etc. Similarly, data is used as a generic term. It can represent files of several Megabytes size, chunks of streaming media packets or contents of web pages.
\cem{A key differentiating
feature of our proposed approach is the focus on how the context information of the captured data can be utilized to maximize the
classification performance of a distributed data mining system.}
\rev{We consider cooperative learners which classify other's data when requested, but instead of maximizing the system utility function, a learner's goal is to maximize its individual utility. However, it can be shown that when the classification costs capture the cost to the learner which is cooperating with another learner to classify its data, maximizing the individual utility corresponds to maximizing the system utility.}

To jointly optimize the performance of the distributed data mining system, we design distributed online learning algorithms whose long-term average rewards converge to the best distributed solution which can be obtained for the classification problem given complete knowledge of online data characteristics as well as their
classification function accuracies and costs when applied to this data. \newc{We adopt the novel cooperative contextual bandit framework we proposed in \cite{cem2013deccontext} to design these algorithms. As a performance measure,
we define the regret as the difference between the expected total reward of the best distributed classification scheme given complete knowledge about classification function accuracies and the expected total reward of the algorithm used by each learner.}
We prove sublinear upper bounds on the regret,
which imply that the average reward converges to the optimal average reward. The upper bound on regret gives a lower bound on convergence rate to the optimal average reward.
\newc{Application of the general framework proposed in \cite{cem2013deccontext} to distributed Big Data mining is not straightforward. In this paper, we address many required innovations for stream mining such as missing labels, delayed labels, asynchronous arrivals, ensemble learners and unsupervised learners who never receive a label but just learn from others. 
}

\newc{Besides the theoretical results, we show that our distributed contextual learning framework can be used to deal with {\em concept drift} \cite{minku2010impact}, which occurs when the distribution of problem instances changes over time. 
Big data applications are often characterized by concept drift,
in which trending topics change rapidly over time.
To illustrate our approach, we provide numerical results by applying our learning algorithms to the classification of network security data and compare the results with
existing state-of-the-art solutions.
}
For example, a network security application needs to analyze several Gigabytes of data generated by different locations and/or at different time in order to detect malicious network behavior (see e.g., \cite{ishibashi2005detecting}). The context in this case can be the time of the day (since the network traffic depends on the time of the day) or it can be the IP address of the machine that sent the data (some locations may be associated with higher malicious activity rate) or context can be two dimensional capturing both the time and the location. 
%
%
In our model, since the classification accuracies are not known a priori, the network security application needs to learn which one to select based on the context information available about the network data. We note that our online learning framework does not require any prior knowledge about the network traffic characteristics or network topology but the security application learns the best actions from its past observations and decisions. In another example, context can be the information about a priori probability about the origin of the data that is sent to the network manager by routers in different locations.

The remainder of the paper is organized as follows. In Section \ref{sec:related}, we describe the related work and highlight the differences from our work. In Section \ref{sec:probform}, we describe the decentralized data classification problem, the optimal distributed classification scheme given the complete system model, its computational complexity, and the regret of a learning algorithm with respect to the optimal classification scheme. 
Then, we consider the model with unknown system statistics and propose distributed online learning algorithms in Section \ref{sec:iid}. 
%
\newc{Several extensions to our proposed learning algorithms are given 
%
%
in Section \ref{sec:discuss}, including concept drift, ensemble learning, operation under privacy and communication constraints.}
Using a network security application we provide numerical results on the performance of our distributed online learning algorithms in Section \ref{sec:numerical}. Finally, the concluding remarks are given in Section \ref{sec:conc}.

\vspace{-0.1in}
\section{Related Work} \label{sec:related}


Online learning in distributed data classification systems aims to address the informational decentralization, communication costs and privacy issues arising in these systems. 
Specifically, in online ensemble learning techniques, the predictions of decentralized and heterogeneous classifiers are combined to improve the classification accuracy.
In these systems, each classifier learns at different rates because either each learner observes the entire feature space but has access to a subset of instances of the entire data set, which is called {\em horizontally distributed} data, or each learner has access to only a subset of the features but the instances can come from the entire data set, which is called {\em vertically distributed} data.  
For example in \cite{predd2006distributed, perez2010robust, breiman1996bagging, wolpert1992stacked}, various solutions are proposed for distributed data mining problems of horizontally distributed data, while in \cite{zheng2011attribute, yubig2013}, ensemble learning techniques are developed that exploit the correlation between the local learners for vertically distributed data. Several cooperative distributed data mining techniques are proposed in \cite{mateos2010distributed, chen2004channel, kargupta1999collective, yubig2013}, where the goal is to improve the prediction accuracy with costly communication between local predictors. In this paper, we take a different approach: instead of focusing on the characteristics of a specific data stream, we focus on the characteristics of data streams with the same context information.
\cem{This novel approach allows us to deal with both horizontally and vertically distributed data in a unified manner within a distributed data mining system.}
%
%
Although our framework and illustrative results are depicted using horizontally distributed data, if context is changed to be the set of relevant features, then our framework and results can operate on vertically distributed data.
Moreover, we assume no prior knowledge of the data and context arrival processes and classification function accuracies, and the learning is done in a non-Bayesian way. \cem{Learning in a non-Bayesian way is appropriate in decentralized system since learners often do not have correct beliefs about the distributed system dynamics.}
%
%

\rev{Most of the prior work in distributed data mining provides algorithms which are asymptotically converging to an optimal or locally-optimal solution without providing any rates of convergence.}
On the contrary, we do not only prove convergence results, but we are also able to explicitly characterize the performance loss incurred at each time step with respect to the optimal solution. In other words, we prove regret bounds that hold uniformly over time. Some of the existing solutions (including \cite{sewell2008ensemble, alpaydin2004introduction, mcconnell2004building, breiman1996bagging, wolpert1992stacked, buhlmann2003boosting, lazarevic2001distributed, perlich2011cross}) propose ensemble learning techniques including bagging, boosting, stacked generalization and cascading, where the goal is to use classification results from several classifiers to increase the prediction accuracy. 
In our work we only consider choosing the best classification function (initially unknown) from a set of classification functions that are accessible by decentralized learners. 
\newc{However, our proposed distributed learning methods can easily be adapted to perform ensemble learning (see Section \ref{sec:ensemble}).}
We provide a detailed comparison to our work in Table \ref{tab:comparison1}.

\comment{Our contextual framework can also deal with concept drift \cite{minku2010impact}. Formally, a concept is the distribution of the problem, \cem{i.e., the joint distribution of the input data stream, true labels and context information,}
at a certain point of time \cite{narasimhamurthy2007framework}. Concept drift is a change in this distribution \cite{gama2004learning, gao2007appropriate}. By treating time as the context, the same regret bounds of our learning algorithms hold under concept drift, without requiring any drift detection mechanism as required in some of the existing solutions employed for dealing with concept drift \cite{baena2006early, minku2012ddd}. Hence, our proposed framework can be used to formalize and solve the problem of concept drift
which was mainly dealt with previously in an ad-hoc manner.}

Other than distributed data mining, our learning framework can be applied to any problem that can be formulated as a decentralized contextual bandit problem \cite{cem2013deccontext}. Contextual bandits have been studied before in \cite{slivkins2009contextual, dudik2011efficient, langford2007epoch, chu2011contextual} in a single agent setting, where the agent sequentially chooses from a set of alternatives with unknown rewards, and the rewards depend on the context information provided to the agent at each time step. 
The main difference of our work from single agent contextual bandits is that: (i) a three phase learning algorithm with {\em training}, {\em exploration} and {\em exploitation} phases are needed instead of the standard two phase, i.e., {\em exploration} and {\em exploitation} phases, algorithms used in centralized contextual bandit problems; (ii) the adaptive partitions of the context space should be formed in a way that each learner can efficiently utilize what is learned by other learners about the same context. \newc{We have provided a detailed discussion of decentralized contextual bandits in \cite{cem2013deccontext}.}

\begin{table}[t]
\centering
{\renewcommand{\arraystretch}{0.6}
{\fontsize{8}{7}\selectfont
\setlength{\tabcolsep}{.1em}
\begin{tabular}{|l|c|c|c|c|c|}
\hline
&  \cite{breiman1996bagging, buhlmann2003boosting, lazarevic2001distributed, chen2004channel, perlich2011cross} & \cite{mateos2010distributed, kargupta1999collective} &  \cite{zheng2011attribute} & This work \\
\hline
Aggregation & non-cooperative & cooperative & cooperative & \rev{no} \\
\hline
Message  & none & data & training  & data and label \\
exchange & & & residual & only if improves  \\
& & & &   performance \\
\hline
Learning  & offline/online & offline & offline & Non-bayesian \\
approach&&&& online\\
\hline
Correlation & N/A & no & no & yes\\
exploitation & & & &\\
\hline
Information from  & no & all & all & only if improves  \\
other learners & & & &  accuracy \\
\hline
Data partition & horizontal & horizontal & vertical & horizontal \\
&&&& and vertical \\
\hline
Bound on regret,  & no &no &no &yes - sublinear\\
convergence rate &&&&\\
\hline
\end{tabular}
}
}
\caption{Comparison with related work in distributed data mining.}
\label{tab:comparison1}
\vspace{-0.2in}
\end{table}

\comment{
\begin{table}[t]
\centering
{\fontsize{8}{6}\selectfont
\setlength{\tabcolsep}{.25em}
\vspace{-0.2in}
\begin{tabular}{|l|c|c|c|c|c|}
\hline
&\cite{slivkins2009contextual, dudik2011efficient, langford2007epoch, chu2011contextual} &  \cite{hliu1, anandkumar, tekin2012sequencing} & \cite{tekin4} & This work \\
\hline
Multi-user & no & yes & yes & yes \\
\hline
Cooperative & N/A & yes & no & yes \\
\hline
Contextual & yes & no & no & yes \\
\hline
Data arrival  & arbitrary & i.i.d. or Markovian & i.i.d. & i.i.d or arbitrary \\
process& & & & \\
\hline
Regret & sublinear & logarithmic & may be linear & sublinear \\
\hline
\end{tabular}
}
\caption{Comparison with related work in multi-armed bandits}
\vspace{-0.35in}
\label{tab:comparison2}
\end{table}
}
\vspace{-0.1in}
\section{Problem Formulation}\label{sec:probform}

The system model is shown in Fig. \ref{fig:system}. There are $M$ learners which are indexed by the set ${\cal M} := \{1,2,\ldots,M\}$.
Let ${\cal M}_{-i} := {\cal M} - \{i\}$ be the set of learners learner $i$ can choose from to send its data for classification.
\cem{These learners work in a discrete time setting $t=1,2,\ldots,T$, where the following events happen sequentially, in each time slot: (i) a data stream $s_i(t)$ with a specific context $x_i(t)$ arrives to each learner $i \in {\cal M}$, (ii) each learner chooses one of its own classification functions or another learner to send its data and context, and produces a label based on the prediction of its own classification function or the learner to which it sent its data and context, (iii) the truth (true label) is revealed eventually, perhaps by events or by a supervisor, only to the learner where the data arrived, \newc{(iv) the learner where the data arrived passes the true label to the learner it had chosen to classify its data, if there is such a learner.}
}
%
%

%
Each learner $i \in {\cal M}$ has access to a set of classification functions ${\cal F}_i$ which it can invoke to classify the data. 
Learner $i$ knows the functions in ${\cal F}_i$ and costs of calling them\footnote{Alternatively, we can assume that the costs are random variables with bounded support whose distribution is unknown. In this case, the learners will not learn the accuracy but they will learn accuracy minus cost.}, but not their accuracies, while it knows the set of other learners ${\cal M}_{-i}$ and costs of calling them but does not know the functions ${\cal F}_{j_i}$, $j_i \in {\cal M}_{-i}$, but only knows an upper bound on the number of classification functions that each learner has, i.e., $F_{\max}$ on $|{\cal F}_{j_i}|$\footnote{For a set $A$, let $|A|$ denote the cardinality of that set.}, $j_i \in {\cal M}_{-i}$.
Let ${\cal K}_i := {\cal F}_i \cup {\cal M}_{-i}$. We call ${\cal K}_i$ the set of arms (alternatives). We use index $k$ to denote any arm in ${\cal K}_i$, $k_i$ to denote the set classification functions of $i$, i.e., the elements of the set ${\cal F}_i$, $j_i$ to denote other learners in ${\cal M}_{-i}$. Let ${\cal F} := \cup_{j \in {\cal M}} {\cal F}_j$ denote the set of all arms of all learners. We use index $k'$ to denote an element of ${\cal F}$.

Learner $i$ can either invoke one of its classification functions or forward the data to another learner to have it labeled. We assume that for learner $i$, calling each classification function $k_i \in {\cal F}_i$ incurs a cost $d^i_{k_i}$.
For example, if the application is delay critical this can be the delay cost, or this can represent the computational cost and power consumption associated with calling a classification function.
We assume that a learner can only call a single function for each input data in order to label it. This is a reasonable assumption when the application is delay sensitive since calling more than one classification function increases the delay. 
\rmv{However, our framework can also be easily extended to ensemble learning where results of different classification function are combined to increase the prediction accuracy. We discuss this possible extension in Section \ref{sec:conc}.}
A learner $i$ can also send its data to another learner in ${\cal M}_{-i}$ in order to have it labeled. Because of the communication cost and the delay caused by processing at the recipient, we assume that whenever a data stream is sent to another learner $j_i \in {\cal M}_{-i}$ a cost of $d^i_{j_i}$ is incurred by learner $i$\footnote{The cost for learner $i$ does not depend on the cost of the classification function chosen by learner $j_i$. Since the learners are cooperative, $j_i$ will obey the rules of the proposed algorithm when choosing a classification function to label $i$'s data.}. Since the costs are bounded, without loss of generality we assume that costs are normalized, i.e., $d^i_k \in [0,1]$ for all $k \in {\cal K}_i$.
\rev{The learners are cooperative which implies that learner $j_i \in {\cal M}_{-i}$ will return a label to $i$ when called by $i$. Similarly, when called by $j_i \in {\cal M}_{-i}$, learner $i$ will return a label to $j_i$. We do not consider the effect of this on $i$'s learning rate, however, since our results hold for the case when other learners are not helping $i$ to learn about its own classification functions, they will hold when other learners help $i$ to learn about its own classification functions. 
If we assume that $d^i_{j_i}$ also captures the cost to learner $j_i$ to classify and send the label back to learner $i$, then maximizing $i$'s own utility corresponds to maximizing the system utility.} 
\rev{
%
}

\begin{figure}
\begin{center}
\includegraphics[width=0.95\columnwidth]{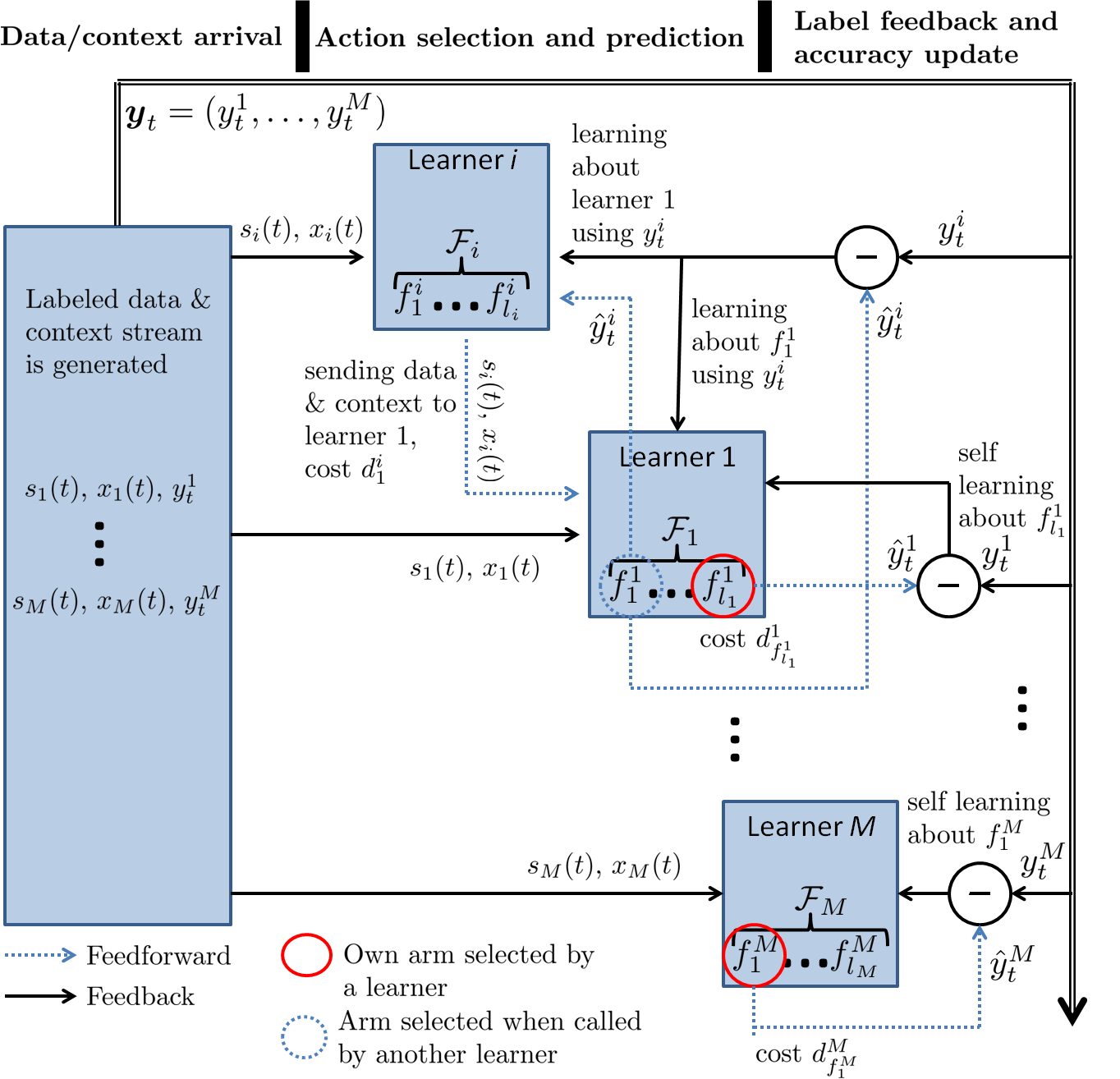}
\vspace{-0.1in}
\caption{Operation of the distributed data classification system during a time slot.} 
\label{fig:system}
\end{center}
\vspace{-0.2in}
\end{figure}

We assume that each \rev{classification function} \com{Give a concrete example of this function} produces a binary label\footnote{In general we can assume that labels belong to $\mathbb{R}$ and define the classification error as the mean squared error or some other metric. Our results can be adapted to this case as well.}. 
%
%
\newc{Considering only binary classifiers is not restrictive since in general, ensembles of binary classifiers can be used to accomplish more complex classification tasks \cite{lienhart2003detector, mao2005multiclass}.}
The data stream at time $t$ arrives to learner $i$ with context information $x_i(t)$. The context may be generated as a result of pre-classification or a header of the data stream.
%
For simplicity we assume that the context space is ${\cal X} = [0,1]^d$, while our results will hold for any bounded $d$ dimensional context space. We also note that the data input is \cem{high} dimensional and its dimension is greater than $d$ (in most of the cases its much larger than $d$) \com{Give example of the dimension of the data in a network security data}. \cem{For example, the network security data we use in numerical results section has 42 features, while the dimension of the context we use is at most 1.} In such a setting, exploiting the context information may significantly improve the classification accuracy while decreasing the classification cost. However, the rate of learning increases with the dimension of the context space, which results in a tradeoff between the rate of learning and the classification accuracy.
Exploiting the context information not only improves the classification accuracy but it can also decrease the classification cost since the context can also provide information about what features to extract from the data. 
\rmv{\cem{For example, in a network security application, the context can be the reputation of the network in which the data originates.
Based on the reputation, a learner may monitor the data stream originating from the network periodically, and this period can decrease with the reputation. This implies that the monitoring cost increases when the reputation of the network decreases.
}
}

Each classification function $k' \in {\cal F}$ has an unknown expected accuracy $\pi_{k'}(x) \in [0,1]$, depending on the context $x$. The accuracy $\pi_{k'}(x)$ represents the probability that 
an input stream with context $x$ will be labeled correctly when classification function $k'$ is used to label it. For a learner $j_i \in {\cal M}_{-i}$ its expected accuracy is equal to the expected accuracy of its best classification function, i.e., $\pi_{j_i}(x) = \max_{k_{j_i} \in {\cal F}_{j_i}} \pi_{k_{j_i}}(x)$.

\rev{Different classification functions can have different accuracies for the same context. Although we do not make any assumptions about the classification accuracy $\pi_k(x)$ and the classification cost $d^i_k$ for $k \in {\cal K}_i$, in general one can assume that classification accuracy increases with classification cost (e.g., classification functions with higher resolution, better processing). In this paper the cost $d^i_k$ is a generic term that can represent any known cost such as processing cost, delay cost, communication cost, etc. 
} 
\rev{We assume that each classification function has similar accuracies for similar contexts; we formalize this in terms of a (uniform) Lipschitz condition.}
\begin{assumption} \label{ass:lipschitz2}
For each $k' \in {\cal F}$, there exists $L>0$, $\alpha>0$ such that for all $x,x' \in {\cal X}$, we have
$|\pi_{k'}(x) - \pi_{k'}(x')| \leq L ||x-x'||^\alpha$,
where $||.||$ denotes the Euclidian norm in $\mathbb{R}^d$.
\end{assumption}

Assumption \ref{ass:lipschitz2} indicates that the accuracy of a classification function for similar contexts will be similar to each other. \rev{Even though the Lipschitz condition can hold with different constants $L_{k'}$ and $\alpha_{k'}$ for each classification function, taking $L$ to be the largest among $L_{k'}$ and $\alpha$ to be the smallest among $\alpha_{k'}$ we get the condition in Assumption \ref{ass:lipschitz2}.}
For example, the context can be the time of the day or/and the location from which the data originates.
Therefore, the relation between the classification accuracy and time can be written down as a Lipschitz condition. \newc{We assume that $\alpha$ is known by the learners, while $L$ does not need to be known. An unknown $\alpha$ can be estimated online using the sample mean estimates of accuracies for similar contexts, and our proposed algorithms can be modified to include the estimation of $\alpha$.}

%

The goal of learner $i$ is to explore the alternatives in ${\cal K}_i$ to learn the accuracies, while at the same time exploiting the best alternative for the context $x_i(t)$ arriving at each time step $t$ that balances the accuracy and cost to minimize its long term loss due to uncertainty. Learner $i$'s problem can be modeled as a contextual bandit problem \cite{slivkins2009contextual, dudik2011efficient, langford2007epoch, chu2011contextual}. After labeling the input at time $t$, each learner observes the true label and updates the sample mean accuracy of the selected arm based on this. 
Accuracies translate into rewards in bandit problems.
In the next subsection, we formally define the benchmark solution which is computed using perfect knowledge about classification accuracies. Then, we define the regret which is the performance loss due to uncertainty about classification accuracies.

\add{\vspace{-0.2in}}
\subsection{Optimal Classification with Complete Information} \label{sec:centralized}

Our benchmark when evaluating the performance of the learning algorithms is the optimal solution which selects the classification function $k'$ with the highest accuracy minus cost for learner $i$ from the set ${\cal F}$ given context $x_i(t)$ at time $t$. We assume that the costs are normalized so the tradeoff between accuracy and cost is captured without using weights. Specifically, the optimal solution we compare against is given by
\add{\vspace{-0.2in}}
\begin{align}
k_i^*(x) = \argmax_{k \in {\cal K}_i} \pi_k(x) - d^i_k, ~~ \forall x \in {\cal X}. \label{eqn:opt2}
\end{align}
%

%
%
%
%
%
%
Knowing the optimal solution means that learner $i$ knows the classification function in ${\cal F}$ that yields the \rev{highest expected accuracy} for each $x \in {\cal X}$. Choosing the best classification function for each context $x$ requires to evaluate the accuracy minus cost for each context and is computationally intractable, because the context space ${\cal X}$ has infinitely many elements. 
\rmv{
Note that the problem remains hard even if we would put additional structure on the optimal classification scheme. For instance, we could assume that the 
%
optimal classification scheme for learner $i$ partitions ${\cal X}$ into  $|\cup_{i \in {\cal M}} {\cal F}_i|$ sets in each of which a single classification function is optimal.
%
Assume for an instance that the data arrival process to learner $i$ is i.i.d. with density $q_i$ and learner $i$ has access to all classification functions $\cup_{i \in {\cal M}} {\cal F}_i$. If learner $i$ knows $q_i$ and the classification accuracies of all classification functions in $\cup_{i \in {\cal M}} {\cal F}_i$, then 
under the above assumption, learner $i$ could compute the optimal classification regions
\begin{align*}
{\cal R}^* = \{R^*_k\}_{(k \in \cup_{i \in {\cal M}} {\cal F}_i)},
\end{align*}
of ${\cal X}$, by solving
\begin{align}
\textrm{{\bf (P1)  }} {\cal R}^* = \argmin_{{\cal R} \in \Theta} E \left[ \left| Y - \sum_{k \in \cup_{i \in {\cal M}} {\cal F}_i} k(X) I(X \in R_l) \right| + \sum_{k \in {\cal F}_i} d_k I(X \in R_k) + \sum_{k \in {\cal M}_{-i}}  d_k I \left(X \in \cup_{j \in {\cal F}_k} R_j \right)   \right], \label{eqn:opt}
\end{align}
where $\Theta$ is the set of $|\cup_{i \in {\cal M}} {\cal F}_i|$-set partitions of ${\cal X}$, the expectation is taken with respect to the distribution $q_i$ and $Y$ is the random variable denoting the true label. Here $I(X \in R_l)$ denotes the event that data received by the learner belongs to the $l$-th set of the partition ${\cal R}$ of ${\cal X}$.
The complexity of finding the optimal classification regions increases exponentially with $|\cup_{i \in {\cal M}} {\cal F}_i|$.
Importantly, note that the learning problem we are trying to solve is even harder than this because the learners are distributed and thus, each learner cannot directly access to all classification functions (learner $i$ only knows set ${\cal F}_i$ and ${\cal M}$, but it does not know any ${\cal F}_j, j \in {\cal M}_{-i}$), and the distributions $q_i$, $i \in {\cal M}$ are unknown (they need not to be i.i.d. or Markovian). Therefore, we use online learning techniques that do not rely on solving the optimization problem in (\ref{eqn:opt}) nor on an estimated version of it. 

\cem{Note that we discussed the above assumption only to illustrate that the optimal solution is still computationally
hard even when we put extra structure on the problem. Our results  
only require Assumption \ref{ass:lipschitz2} to hold.}
}
 
%
%
\vspace{-0.1in}
\subsection{The Regret of Learning}

In this subsection we define the regret as a performance measure of the learning algorithm used by the learners. Simply, the regret is the loss incurred due to the unknown system dynamics. Regret of a learning algorithm $\alpha$ which selects an arm $\alpha_t(x_i(t))$ at time $t$ for learner $i$ is defined with respect to the best arm $k_i^*(x)$ given in (\ref{eqn:opt2}).
The regret of a learning algorithm for learner $i$ is given by
\add{\vspace{-0.05in}}
\begin{align*}
R_i(T) &:= \sum_{t=1}^T \left( \pi_{k_i^*(x_i(t))}(x_i(t)) - d^i_{k_i^*(x_i(t))} \right) \\ 
&- E \left[ \sum_{t=1}^T ( I(\hat{y}^i_t(\alpha_t(x_i(t))) = y^i_t) - d^i_{\alpha_t(x_i(t))}) \right] ,
\end{align*}
where $\hat{y}^i_t(.)$ denotes the prediction of the arm selected by learner $i$ at time $t$, $y^i_t$ denotes the true label of the data stream that arrived to learner $i$ in time slot $t$, \newc{and the expectation is taken with respect to the randomness of the prediction.} Regret gives the convergence rate of the total expected reward of the learning algorithm to the value of the optimal solution given in (\ref{eqn:opt2}). Any algorithm whose regret is sublinear, i.e., $R_i(T) = O(T^\gamma)$ such that $\gamma<1$, will converge to the optimal solution in terms of the average reward.

In the next section, we propose two online learning algorithms which achieves sublinear regret for the distributed classification problem. \newc{Detailed analysis of these algorithms is given in \cite{cem2013deccontext}. In this paper we only briefly mention these algorithms and focus instead on the specific challenges involved in applying these algorithms to Big Data mining.}



\vspace{-0.15in}
\section{Distributed Online Learning Algorithms for Big Data Mining} \label{sec:iid}

In this section we propose two online learning algorithms for Big Data mining. The first algorithm is {\em Classify or Send for classification} (CoS) whose pseudocode is given in Fig. \ref{fig:COS}. Basically, CoS forms a uniform partition ${\cal P}_T$ of the context space consisting of $(m_T)^d$, $d$-dimensional hypercubes, where the $l$th hypercube is denoted by $P_l$, and $m_T$ is called the slicing parameter which depends on final time $T$. Each of these hypercubes are treated as separate bandit problems where the goal for each problem is to learn the arm in ${\cal K}_i$ that yields the highest accuracy minus cost. Different from the single-agent contextual bandits, since the context arrivals to different learners are different, a training phase in addition to exploration and exploitation phases are required to learn the accuracies of the other learners correctly. \newc{In order to decide when to train, explore or exploit, CoS keeps three control functions $D_1(t)$, $D_2(t)$ and $D_3(t)$, and two different sets of counters $N^i_{k,l}(t)$ for $k \in {\cal K}_i$, $N^i_{1,k,l}(t)$ for $k \in {\cal M}_{-i}$ for all $P_l \in {\cal P}_T$. Let
\begin{align*}
{\cal S}_{i,l}(t) := &\left\{ k_i \in {\cal F}_i \textrm{ such that } N^i_{k_i,l}(t) \leq D_1(t)  \textrm{ or } j_i \in {\cal M}_{-i} \right. \\
&\left. \textrm{ such that } N^i_{1,j_i,l}(t) \leq D_2(t) \textrm{ or } N^i_{j_i,l}(t) \leq D_3(t)   \right\}.
\end{align*} 
At time slot $t$ if $x_i(t) \in P_l$ and ${\cal S}_{i,l}(t) = \emptyset$, then CoS exploits by choosing the arm in ${\cal K}_i$ with the highest empirical reward $\bar{r}^i_{k,l}(t)$ (i.e, sample mean accuracy minus the cost) for set $P_l$. Otherwise it either trains a learner in ${\cal M}_{-i}$ or explores an arm in ${\cal K}_i$. Due to the high uncertainty about the rewards collected during the training phases, they are not used to calculate the empirical reward. 
}
The pseudocodes for each phase is given in Fig. \ref{fig:mtrain}. 

\begin{figure}[htb]
\fbox {
\begin{minipage}{0.95\columnwidth}
{\fontsize{8}{7}\selectfont
\flushleft{Classify or Send for Classification (CoS for learner $i$):}
\begin{algorithmic}[1]
\STATE{Input: $D_1(t)$, $D_2(t)$, $D_3(t)$, $T$, $m_T$}
\STATE{Initialize: Partition $[0,1]^d$ into $(m_T)^d$ sets. Let ${\cal P}_T = \{ P_1, \ldots, P_{(m_T)^d} \}$ denote the sets in this partition. $N^i_{k,l}=0, \forall k \in {\cal K}_i, P_l \in {\cal P}_T$, $N^i_{1,k,l}=0, \forall k \in {\cal M}_{-i}, P_l \in {\cal P}_T$.}
\WHILE{$t \geq 1$}
\FOR{$l=1,\ldots,(m_T)^d$}
\IF{$x_i(t) \in P_l$}
\IF{$\exists k \in {\cal F}_i \textrm{ such that } N^i_{k,l} \leq D_1(t)$}
\STATE{Run {\bf Explore}($k$, $N^i_{k,l}$, $\bar{r}_{k,l}$)}
\ELSIF{$\exists k \in {\cal M}_{-i} \textrm{ such that } N^i_{1,k,l} \leq D_2(t)$}
\STATE{Obtain $N^k_l(t)$ from $k$, set $N^i_{1,k,l} = N^k_l(t) - N^i_{k,l}$}
\IF{$N^i_{1,k,l} \leq D_2(t)$}
\STATE{Run {\bf Train}($k$, $N^i_{1,k,l}$)}
\ELSE
\STATE{Go to line 15}
\ENDIF
\ELSIF{$\exists k \in {\cal M}_{-i} \textrm{ such that } N^i_{k,l} \leq D_3(t)$}
\STATE{Run {\bf Explore}($k$, $N^i_{k,l}$, $\bar{r}_{k,l}$)}
\ELSE
\STATE{Run {\bf Exploit}($\boldsymbol{M}^i_l$, $\bar{\boldsymbol{r}}_l$, ${\cal K}_i$)}
\ENDIF
\ENDIF
\ENDFOR
\STATE{$t=t+1$}
\ENDWHILE
\end{algorithmic}
}
\end{minipage}
} \caption{Pseudocode for the CoS algorithm.} \label{fig:COS}
\add{\vspace{-0.12in}}
\end{figure}
\begin{figure}[htb]
\fbox {
\begin{minipage}{0.95\columnwidth}
{\fontsize{8}{7}\selectfont
{\bf Train}($k$, $n$):
\begin{algorithmic}[1]
\STATE{Select arm $k$, receive prediction $\hat{y}(k)$. Receive reward $r_k(t) = I(\hat{y}(k) = y_t) - d^i_{k}$. $n++$. }
\end{algorithmic}
{\bf Explore}($k$, $n$, $r$):
\begin{algorithmic}[1]
\STATE{Select arm $k$, receive prediction $\hat{y}(k)$. Receive reward $r_k(t) = I(\hat{y}(k) = y_t) - d^i_{k}$. $r = \frac{n r + r_k(t)}{n + 1}$. $n++$.   }
\end{algorithmic}
{\bf Exploit}($\boldsymbol{n}$, $\boldsymbol{r}$, ${\cal K}_i$):
\begin{algorithmic}[1]
\STATE{Select arm $k\in \argmax_{j \in {\cal K}_i} r_j$, receive prediction $\hat{y}(k)$. Receive reward $r_k(t) = I(\hat{y}(k) = y_t) - d^i_{k}$. $\bar{r}_{k} = \frac{n_k \bar{r}_{k} + r_k(t)}{n_k + 1}$. $n_k++$. }
\end{algorithmic}
}
\end{minipage}
} \caption{Pseudocode of the training, exploration and exploitation modules.} \label{fig:mtrain}
\add{\vspace{-0.2in}}
\end{figure}

\comment{
\begin{figure}[htb]
\fbox {
\begin{minipage}{0.95\columnwidth}
{\fontsize{9}{9}\selectfont
{\bf Explore}($k$, $n$, $r$):
\begin{algorithmic}[1]
\STATE{select arm $k$}
\STATE{Receive reward $r_k(t) = I(k(x_i(t)) = y_t) - d_{k(x_i(t))}$}
\STATE{$r = \frac{n r + r_k(t)}{n + 1}$}
\STATE{$n++$}
\end{algorithmic}
}
\end{minipage}
} \caption{Pseudocode of the exploration module} \label{fig:mexplore}
\end{figure}
}

\comment{
\begin{figure}[htb]
\fbox {
\begin{minipage}{0.9\columnwidth}
{\fontsize{9}{9}\selectfont
{\bf Exploit}($\boldsymbol{n}$, $\boldsymbol{r}$, ${\cal K}_i$):
\begin{algorithmic}[1]
\STATE{select arm $k\in \argmax_{j \in {\cal K}_i} r_j$}
\STATE{Receive reward $r_k(t) = I(k(x_i(t)) = y_t) - d_{k(x_i(t))}$}
\STATE{$\bar{r}_{k} = \frac{n_k \bar{r}_{k} + r_k(t)}{n_k + 1}$}
\STATE{$n_k++$}
\end{algorithmic}
}
\end{minipage}
} \caption{Pseudocode of the exploitation module} \label{fig:mexploit}
\end{figure}
}

Our second algorithm is the {\em distributed context zooming} algorithm (DCZA) whose pseudocode is given in Fig. \ref{fig:DDZA}. The difference of DCZA from CoS is that instead of starting with a uniform partition of the context space, it adaptively creates partition of the context space based on the context arrival process. \newc{It does this by splitting a level $l$ hypercube in the partition of the context space into $2^d$ level $l+1$ hypercubes with equal sizes, when the number of context arrivals to the level $l$ hypercube exceeds $A 2^{pl}$ for constants $A,p>0$.}

We provide a detailed discussion of the operation of these algorithms and comparison of them in terms of their performance and computational requirements under different context arrival processes in \cite{cem2013deccontext}. All the theorems we derived for CoS (CLUP in \cite{cem2013deccontext}) and DCZA also holds for this paper as well. Due to limited space, we do not rewrite these theorems here. Our focus in this paper is to consider different aspects of the application of these algorithms to Big Data mining, and provide analytical and numerical results for them.
\vspace{-0.1in}
\begin{figure}[htb]
\fbox {
\begin{minipage}{0.95\columnwidth}
{\fontsize{8}{7}\selectfont
\flushleft{Distributed Context Zooming Algorithm (DCZA for learner $i$):}
\begin{algorithmic}[1]
\STATE{Input: $D_1(t)$, $D_2(t)$, $D_3(t)$, $p$, $A$}
\STATE{Initialization: ${\cal A} = \{[0,1]^d\}$, Run {\bf Initialize}(${\cal A}$)}
\STATE{Notation: $\boldsymbol{M}^i_C := (N^i_{k,c})_{k \in {\cal K}_i}$, 
$\bar{\boldsymbol{r}}_C := (\bar{r}_{k,C})_{k \in {\cal K}_i}$, $l_C$: level of hypercube $C$.}
\WHILE{$t \geq 1$}
\FOR{$C \in {\cal A}$}
\IF{$x_i(t) \in C$}
\IF{$\exists k \in {\cal F}_i \textrm{ such that } N^i_{k,C} \leq D_1(t)$}
\STATE{Run {\bf Explore}($k$, $N^i_{k,C}$, $\bar{r}_{k,C}$)}
\ELSIF{$\exists k \in {\cal M}_{-i} \textrm{ such that } N^i_{1,k,C} \leq D_2(t)$}
\STATE{Obtain $N^k_C(t)$ from $k$}
\IF{$N^k_C(t)=0$}
\STATE{ask $k$ to create hypercube $C$, set $N^i_{1,k,C}=0$}
\ELSE
\STATE{set $N^i_{1,k,C} = N^k_C(t) - N^i_{k,C}$}
\ENDIF
\IF{$N^i_{1,k,C} \leq D_2(t)$}
\STATE{Run {\bf Train}($k$, $N^i_{1,k,C}$)}
\ELSE
\STATE{Go to line 21}
\ENDIF
\ELSIF{$\exists k \in {\cal M}_{-i} \textrm{ such that } N^i_{k,C} \leq D_3(t)$}
\STATE{Run {\bf Explore}($k$, $N^i_{k,C}$, $\bar{r}_{k,C}$)}
\ELSE
\STATE{Run {\bf Exploit}($\boldsymbol{M}^i_C$, $\bar{\boldsymbol{r}}_C$, ${\cal K}_i$)}
\ENDIF
\ENDIF
\STATE{$N^i_C = N^i_C +1$}
\IF{$N^i_C \geq A 2^{pl_C}$}
\STATE{Create $2^d$ level $l_C+1$ child hypercubes denoted by ${\cal A}^{l_C+1}_C$}
\STATE{Run {\bf Initialize}(${\cal A}^{l_C+1}_C$)}
\STATE{${\cal A} = {\cal A} \cup {\cal A}^{l_C+1}_C - C$}
\ENDIF
\ENDFOR
\STATE{$t=t+1$}
\ENDWHILE
\end{algorithmic}
}
\end{minipage}
} 
\fbox {
\begin{minipage}{0.95\columnwidth}
{\fontsize{8}{7}\selectfont
{\bf Initialize}(${\cal B}$):
\begin{algorithmic}[1]
\FOR{$C \in {\cal B}$}
\STATE{Set $N^i_C = 0$, $N^i_{k,C}=0$, $\bar{r}_{k,C}=0$ for $C \in {\cal A}, k \in {\cal K}_i$, $N^i_{1,k,C}=0$ for $k \in {\cal M}_{-i}$}
\ENDFOR
\end{algorithmic}
}
\end{minipage}
}
\caption{Pseudocode of the DCZA algorithm and its initialization module.} \label{fig:DDZA}
\vspace{-0.1in}
\end{figure}

\remove{
\begin{align*}
R(T) = E[R_e(T)] + E[R_s(T)] + E[R_n(T)],
\end{align*}
}
\remove{
\begin{proof}
Let $\Omega$ denote the space of all possible outcomes, and $w$ be a sample path. The event that the algorithm exploits at time $t$ is given by
\begin{align*}
{\cal W}^i_{l}(t) := \{ w : S_{i,l}(t) = \emptyset  \}.
\end{align*}
We will bound the probability that the algorithm chooses a suboptimal arm in an exploitation step. Using that we can bound the expected number of times a suboptimal arm is chosen by the algorithm. Note that every time a suboptimal arm is chosen, since $\pi_k(x) - d_k \in [-1,1]$, the loss is at most $2$. Therefore $2$ times the expected number of times a suboptimal arm is chosen in an exploitation step bounds the regret due to suboptimal choices in exploitation steps.
Let ${\cal V}^i_{k,l}(t)$ be the event that a suboptimal action $k$ is chosen at time $t$. We have
\begin{align*}
R_s(T) &\leq \sum_{l \in {\cal P}_T} \sum_{t=1}^T \sum_{k \in {\cal L}^i_\theta(t)} I({\cal V}^i_{k,l}(t), {\cal W}^i_{l}(t) ).
\end{align*}
Taking the expectation
\begin{align}
E[R_s(T)] \leq \sum_{l \in {\cal P}_T} \sum_{t=1}^T \sum_{k \in {\cal L}^i_\theta(t)} P({\cal V}^i_{k,l}(t), {\cal W}^i_{l}(t) ) \label{eqn:subregret}
\end{align}

Let ${\cal B}^i_{k,l}(t)$ be the event that at most $t^{\phi}$ samples in ${\cal E}^i_{k,l}(t)$ are collected from suboptimal classification functions of the $k$-th arm. Obviously for any $k \in {\cal F}_i$, ${\cal B}^i_{k,l}(t) = \Omega$, while this is not always true for $k \in {\cal M}_{-i}$. 
We have
\begin{align}
\{ {\cal V}^i_{k,l}(t), {\cal W}^i_{l}(t)\} &\subset \left\{ \bar{r}_{k,l}(t) \geq \bar{r}_{k^*(l),l}(t), {\cal W}^i_{l}(t), {\cal B}^i_{k,l}(t) \right\} 
\cup \left\{ \bar{r}_{k,l}(t) \geq \bar{r}_{k^*(l),l}(t), {\cal W}^i_{l}(t), {\cal B}^i_{k,l}(t)^c \right\} \notag \\
&\subset \left\{ \bar{r}_{k,l}(t) \geq \overline{\mu}_{k,l} + H_t, {\cal W}^i_{l}(t), {\cal B}^i_{k,l}(t)  \right\} 
\cup \left\{ \bar{r}_{k^*(l),l}(t) \leq \underline{\mu}_{k^*(l),l} - H_t, {\cal W}^i_{l}(t), {\cal B}^i_{k,l}(t)  \right\} \notag \\
& \cup \left\{ \bar{r}_{k,l}(t) \geq \bar{r}_{k^*(l),l}(t), 
\bar{r}_{k,l}(t) < \overline{\mu}_{k,l} + H_t,
 \bar{r}_{k^*(l),l}(t) > \underline{\mu}_{k^*(l),l} - H_t,
{\cal W}^i_{l}(t) ,{\cal B}^i_{k,l}(t)  \right\} \notag \\
&\cup {\cal B}^i_{k,l}(t)^c , \label{eqn:vkt}
\end{align}
for some $H_t >0$. This implies that 
\begin{align}
P \left( {\cal V}^i_{k,l}(t), {\cal W}^i_{l}(t) \right) 
&\leq P \left( \bar{r}_{k,l}(t) \geq \overline{\mu}_{k,l} + H_t, {\cal W}^i_{l}(t), {\cal B}^i_{k,l}(t)  \right) \notag  \\
&+ P \left( \bar{r}_{k^*(l),l}(t) \leq \underline{\mu}_{k^*(l),l} - H_t, {\cal W}^i_{l}(t), {\cal B}^i_{k,l}(t) \right) \notag  \\
&+ P \left( \bar{r}_{k,l}(t) \geq \bar{r}_{k^*(l),l}(t), 
\bar{r}_{k,l}(t) < \overline{\mu}_{k,l} + H_t,
 \bar{r}_{k^*(l),l}(t) > \underline{\mu}_{k^*(l),l} - H_t,
{\cal W}^i_{l}(t), {\cal B}^i_{k,l}(t)  \right) \notag \\
&+ P({\cal B}^i_{k,l}(t)^c). \label{eqn:ubound1}
\end{align}
We have for any suboptimal arm $k \in {\cal K}_i$
\begin{align}
& P \left( \bar{r}_{k,l}(t) \geq \bar{r}_{k^*(l),l}(t), 
\bar{r}_{k,l}(t) < \overline{\mu}_{k,l} + H_t,
 \bar{r}_{k^*(l),l}(t) > \underline{\mu}_{k^*(l),l} - H_t,
{\cal W}^i_{l}(t), {\cal B}^i_{k,l}(t)  \right) \notag \\
&\leq P \left( \bar{r}^{\textrm{best}}_{k,l}(|{\cal E}^i_{k,l}(t)|) 
\geq \bar{r}^{\textrm{worst}}_{k^*(l),l}(|{\cal E}^i_{k^*(l),l}(t)|)
-  t^{\phi-1} , 
\bar{r}^{\textrm{best}}_{k,l}(|{\cal E}^i_{k,l}(t)|) < \overline{\mu}_{k,l} + L \left( \frac{\sqrt{d}}{m_T} \right)^\alpha + H_t +  t^{\phi-1}, \right. \notag \\
& \left. \bar{r}^{\textrm{worst}}_{k^*(l),l}(|{\cal E}^i_{k^*(l),l}(t)|) > \underline{\mu}_{k^*(l),l} - L \left( \frac{\sqrt{d}}{m_T} \right)^\alpha - H_t,
{\cal W}^i_{l}(t)    \right). \notag
\end{align}
Since $k$ is a suboptimal arm, when
\begin{align}
2 L \left( \frac{\sqrt{d}}{m_T} \right)^\alpha + 2H_t + 2t^{\phi-1} - a_1 t^\theta \leq 0,
\label{eqn:boundcond}
\end{align}
the three inequalities given below
\begin{align*}
& \underline{\mu}_{k^*(l),l} - \overline{\mu}_{k,l} > a_1 t^{\theta},\\
& \bar{r}^{\textrm{best}}_{k,l}(|{\cal E}^i_{k,l}(t)|) < \overline{\mu}_{k,l} + L \left( \frac{ \sqrt{d}}{m_T} \right)^\alpha + H_t + t^{\phi-1} ,\\
& \bar{r}^{\textrm{worst}}_{k^*(l),l}(|{\cal E}^i_{k,l}(t)|) > \underline{\mu}_{k^*(l),l} - L \left( \frac{ \sqrt{d}}{m_T} \right)^\alpha - H_t,
\end{align*}
together imply that 
\begin{align*}
\bar{r}^{\textrm{best}}_{k,l}(|{\cal E}^i_{k,l}(t)|) < \bar{r}^{\textrm{worst}}_{k^*(l),l}(|{\cal E}^i_{k,l}(t)|) -  t^{\phi-1} ,
\end{align*}
which implies that for a suboptimal arm $k \in {\cal K}_i$, we have
\begin{align}
P \left( \bar{r}_{k,l}(t) \geq \bar{r}_{k^*(l),l}(t), 
\bar{r}_{k,l}(t) < \overline{\mu}_{k,l} + H_t,
 \bar{r}_{k^*(l),l}(t) > \underline{\mu}_{k^*(l),l} - H_t,
{\cal W}^i_{l}(t), {\cal B}^i_{k,l}(t)  \right) = 0. \label{eqn:vktbound1}
\end{align}
Let $H_t = 2 t^{\phi-1}$. Then a sufficient condition that implies (\ref{eqn:boundcond}) is
\begin{align}
&2 L( \sqrt{d})^\alpha t^{- \gamma \alpha} + 6 t^{\phi-1} \leq a_1 t^\theta. \label{eqn:maincondition}
\end{align}
Assume that (\ref{eqn:maincondition}) holds for all $t \geq 1$.
Using a Chernoff-Hoeffding bound, for any $k \in {\cal L}^i_{\theta}(t)$, since on the event ${\cal W}^i_{l}(t)$, $|{\cal E}^i_{k,l}(t)| \geq t^z \log t$, we have
\begin{align}
P \left( \bar{r}_{k,l}(t) \geq \overline{\mu}_{k,l} + H_t, {\cal W}^i_{l}(t), {\cal B}^i_{k,l}(t) \right) 
&\leq P \left( \bar{r}^{\textrm{best}}_{k,l}(|{\cal E}^i_{k,l}(t)|) \geq \overline{\mu}_{k,l} + H_t, {\cal W}^i_{l}(t) \right) \notag \\
&\leq e^{-2 (H_t)^2 t^z \log t}  = e^{-8 t^{2\phi-2} t^z \log t} ~, \label{eqn:vktbound2}
\end{align}
and
\begin{align}
&P \left( \bar{r}_{k^*(l),l}(t) \leq \underline{\mu}_{k^*(l),l} - H_t, {\cal W}^i_{l}(t), {\cal B}^i_{k,l}(t) \right) \notag \\
&\leq P \left( \bar{r}^{\textrm{worst}}_{k^*(l),l}(|{\cal E}^i_{k^*(l),l}(t)|)  \leq \underline{\mu}_{k^*(l),l} - H_t +  t^{\phi-1}, {\cal W}^i_{l}(t) \right) \notag \\
&\leq e^{-2 (H_t -  t^{\phi-1})^2 t^z \log t} = e^{-2 t^{2\phi-2} t^z \log t}. \label{eqn:vktbound3}
\end{align}
In order to bound the regret, we will sum (\ref{eqn:vktbound2}) and (\ref{eqn:vktbound3}) for all $t$ up to $T$. For regret to be small we want the sum to be sublinear in $T$. This holds when $2\phi -2 +z \geq 0$. We want $z$ to be small since regret due to explorations increases with $z$, and we also want $\phi$ to be small since we will show that our regret bound increases with $\phi$. Therefore we set $2\phi -2 +z =0$, hence 
\begin{align}
\phi = 1-z/2. \label{eqn:maincondition2}
\end{align}
When (\ref{eqn:maincondition2}) holds we have
\begin{align}
P \left( \bar{r}_{k,l}(t) \geq \overline{\mu}_{k,l} + H_t, {\cal W}^i_{l}(t), {\cal B}^i_{k,l}(t) \right) \leq \frac{1}{t^2}, \label{eqn:vktbound22}
\end{align}
and
\begin{align}
P \left( \bar{r}_{k^*(l),l}(t) \leq \underline{\mu}_{k^*(l),l} - H_t, {\cal W}^i_{l}(t), {\cal B}^i_{k,l}(t) \right) \leq \frac{1}{t^2}. \label{eqn:vktbound32}
\end{align}

Finally, for $k \in {\cal F}_i$ obviously we have $P({\cal B}^i_{k,l}(t)^c)=0$. For $k \in {\cal M}_{-i}$, let $X^i_{k,l}(t)$ denote the random variable which is the number of times a suboptimal classification function for arm $k$ is chosen in exploitation steps when the context is in set $P_l$ by time $t$. We have $\{ {\cal B}^i_{k,l}(t)^c, {\cal W}^i_l(t)  \} = \{ X^i_{k,l}(t) \geq t^\phi \}$. Applying the Markov inequality we have
\begin{align*}
P({\cal B}^i_{k,l}(t)^c, {\cal W}^i_l(t)) \leq \frac{E[X^i_{k,l}(t)]}{t^\phi},
\end{align*}
Let $\Xi^i_{k,l}(t)$ be the event that a suboptimal classification function $m \in {\cal F}_k$ is called by learner $k \in {\cal M}_{-i}$, when it is invoked by learner $i$ for the $t$-th time in the exploitation phase of learner $i$. 
We have 
\begin{align*}
X^i_{k,l}(t) = \sum_{t'=1}^{{\cal E}^i_{k,l}(t)} I(\Xi^i_{k,l}(t')),
\end{align*}
and
\begin{align*}
P \left( \Xi^i_{k,l}(t) \right) 
&\leq \sum_{m \in {\cal L}^k_\theta} P \left( \bar{r}_{m,l}(t) \geq \bar{r}^{*k}_{l}(t) \right) \\
&\leq \sum_{m \in {\cal L}^k_\theta}
\left(  P \left( \bar{r}_{m,l}(t) \geq \overline{\mu}_{m,l} + H_t, {\cal W}^i_{l}(t) \right)   
+ P \left( \bar{r}^{*k}_{l}(t) \leq \underline{\mu}^{*k}_{l} - H_t, {\cal W}^i_{l}(t) \right) \right. \\
&\left. + P \left( \bar{r}_{m,l}(t) \geq \bar{r}^{*k}_{l}(t), 
\bar{r}_{m,l}(t) < \overline{\mu}_{m,l} + H_t,
 \bar{r}^{*k}_{l}(t) > \underline{\mu}^{*k}_{l} - H_t ,
{\cal W}^i_{l}(t) \right)  \right).
\end{align*}
When (\ref{eqn:maincondition}) holds, since $\phi = 1 - z/2$, the last probability in the sum above is equal to zero while the first two inequalities are upper bounded by $e^{-2(H_t)^2 t^z \log t}$. This is due to the second phase of the exploration algorithm which requires at least $t^z \log t$ samples from the second exploration phase for all learners before the algorithm exploits any learner. Therefore, we have
\begin{align*}
P \left( \Xi^i_{k,l}(t) \right) \leq \sum_{m \in {\cal L}^k_\theta} 2 e^{-2(H_t)^2 t^z \log t} \leq \frac{2 |{\cal F}_k|}{t^2}.
\end{align*}
These together imply that 
\begin{align*}
E[X^i_{k,l}(t)] \leq \sum_{t'=1}^{\infty} P(\Xi^i_{k,l}(t')) \leq 2 |{\cal F}_k| \sum_{t'=1}^\infty \frac{1}{t^2}.
\end{align*}
Therefore from the Markov inequality we get
\begin{align}
P({\cal B}^i_{k,l}(t)^c, {\cal W}^i_l(t)) = P(X^i_{k,l}(t) \geq t^\phi) \leq \frac{2 |{\cal F}_k| \beta_2}{t^{1-z/2}}. \label{eqn:selectionbound}
\end{align}
Then using (\ref{eqn:vktbound1}), (\ref{eqn:vktbound22}), (\ref{eqn:vktbound32}) and (\ref{eqn:selectionbound}) we have 
\begin{align*}
P \left( {\cal V}^i_{k,l}(t), {\cal W}^i_l(t)  \right) \leq \frac{2}{t^{2}} + \frac{2 |{\cal F}_k| \beta_2}{t^{1-z/2}},
\end{align*}
for any $k \in {\cal M}_{-i}$, and
\begin{align*}
P \left( {\cal V}^i_{k,l}(t), {\cal W}^i_l(t)  \right) \leq \frac{2}{t^{2}},
\end{align*}
for any $k \in {\cal F}_i$. By (\ref{eqn:subregret}), we have
\begin{align}
E[R_s(T)] &\leq 2^d T^{\gamma d}
 \left( 2 (M-1+|{\cal F}_i|) \beta_2  + 2 (M-1) F_{\max} \beta_2 \sum_{t=1}^T \frac{1}{t^{1-z/2}} \right) \\
&\leq  2^{d+1} (M-1+|{\cal F}_i|) \beta_2 T^{\gamma d}+ \frac{2^{d+2} (M-1)  F_{\max} \beta_2}{z} T^{\gamma d + z/2}, \label{eqn:regret_s}
\end{align}
where (\ref{eqn:regret_s}) follows from Appendix \ref{app:seriesbound}. 
\end{proof}
}

%
%
\remove{
\begin{proof}
Let $\Xi^i_{k,l}(t)$ be the event that a suboptimal classification function $m \in {\cal F}_k$ is called by learner $k \in {\cal M}_{-i}$, when it is invoked by learner $i$ for the $t$-th time in the exploitation phase of learner $i$. 
We have 
\begin{align*}
X^i_{k,l}(t) = \sum_{t'=1}^{{\cal E}^i_{k,l}(t)} I(\Xi^i_{k,l}(t')),
\end{align*}
and
\begin{align*}
P \left( \Xi^i_{k,l}(t) \right) 
&\leq \sum_{m \in {\cal L}^k_\theta} P \left( \bar{r}_{m,l}(t) \geq \bar{r}^{*k}_{l}(t) \right) \\
&\leq \sum_{m \in {\cal L}^k_\theta}
\left(  P \left( \bar{r}_{m,l}(t) \geq \overline{\mu}_{m,l} + H_t, {\cal W}^i_{l}(t) \right)   
+ P \left( \bar{r}^{*k}_{l}(t) \leq \underline{\mu}^{*k}_{l} - H_t, {\cal W}^i_{l}(t) \right) \right. \\
&\left. + P \left( \bar{r}_{m,l}(t) \geq \bar{r}^{*k}_{l}(t), 
\bar{r}_{m,l}(t) < \overline{\mu}_{m,l} + H_t,
 \bar{r}^{*k}_{l}(t) > \underline{\mu}^{*k}_{l} - H_t ,
{\cal W}^i_{l}(t) \right)  \right).
\end{align*}
Let $H_t = 2 t^{-z/2}$. Similar to the proof of Lemma \ref{lemma:suboptimal1}, the last probability in the sum above is equal to zero while the first two inequalities are upper bounded by $e^{-2(H_t)^2 t^z \log t}$. This is due to the second phase of the exploration algorithm which requires at least $t^z \log t$ samples from the second exploration phase for all learners before the algorithm exploits any learner. Therefore, we have
\begin{align*}
P \left( \Xi^i_{k,l}(t) \right) \leq \sum_{m \in {\cal L}^k_\theta} 2 e^{-2(H_t)^2 t^z \log t} \leq \frac{2 |{\cal F}_k|}{t^2}.
\end{align*}
These together imply that 
\begin{align*}
E[X^i_{k,l}(t)] \leq \sum_{t'=1}^{\infty} P(\Xi^i_{k,l}(t')) \leq 2 |{\cal F}_k| \sum_{t'=1}^\infty \frac{1}{t^2}.
\end{align*}
\end{proof}
}
\remove{
\begin{proof}
If a near optimal arm in ${\cal F}_i$ is chosen at time $t$, the contribution to the regret is at most $a_1 t^{\theta}$. If a near optimal arm in $k \in {\cal M}_{-i}$ is chosen at time $t$, and if $k$ classifies according to one of its near optimal classification functions than the contribution to the regret is at most $2 a_1 t^{\theta}$.
Therefore the total regret due to near optimal arm selections in ${\cal K}_i$ by time $T$ is upper bounded by 
\begin{align*}
2 a_1 \sum_{t=1}^T t^{\theta} & \leq \frac{2 a_1 T^{1+\theta}}{1+\theta},
\end{align*}
by using the result in Appendix \ref{app:seriesbound}. 
Each time a near optimal arm in $k \in {\cal M}_{-i}$ is chosen in an exploitation step, there is a small probability that the classification function called by arm $k$ is a suboptimal one. Given in Lemma \ref{lemma:callother}, the expected number of times a suboptimal classification function is called is bounded by $2 |{\cal F}_k| \beta_2$. Each time a suboptimal classification function is chosen the regret can be at most $2$.
\end{proof}
}

In \cite{cem2013deccontext}, we used the context dimension $d$ as an input parameter and compared with the optimal solution given a fixed $d$. However, the context information can also be adaptively chosen over time. 
\newc{For example, in network security, the context can be either time of the day, origin of the data or both. The classifier accuracies will depend on what is used as context information. A detailed discussion of adaptively choosing the context is given in Section \ref{sec:discuss}.}
%
%
Remarks about computational complexity and memory requirements of CoS and DCZA can be found in \cite{cem2013deccontext}.

\newc{In the following subsections, we discuss three important aspects of online learning in data mining systems. The first is about the classification functions which learn online and improve their accuracies over time, instead of having fixed accuracies. The second is about delayed feedback. The third one is about the case when the true label is not always available, and the fourth one considers how explorations and trainings can be reduced. We present all of these aspects considering one of the two algorithms, but the same approach can also be applied to both algorithms.
}

%

\vspace{-0.1in}
\subsection{Online learning classification functions}

In our analysis we assumed that given a context $x$, the classification function accuracy $\pi_{k'}(x)$ is fixed. This holds when the classification functions are trained a priori, but the learners do not know the accuracy because $k'$ is not tested yet. By using our contextual framework, we can also allow the classification functions to learn over time based on the data. \rev{Usually in Big Data applications we cannot have the classifiers being pre-trained as they are often deployed for the first time in a certain setting. For example in \cite{chai2002bayesian}, Bayesian online classifiers are used for text classification and filtering.}
%
We do this by introducing time as a context, thus increasing the context dimension to $d+1$. Time is normalized in interval $[0,1]$ such that $0$ corresponds to $t=0$, 1 corresponds to $t=T$ and each time slot is an interval of length $1/T$. 
For an online learning classification function, intuitively the accuracy is expected to increase with the number of samples, and thus, $\pi_{k'}(x, t)$ will be non-decreasing in time for $k' \in {\cal F}$. 
\newc{On the other hand, when the true label is received and the classification function is updated, it can still make errors. Usually the increase in classification accuracy after a single update is bounded.
Based on these observations, we assume that the following assumption which is a variant of Assumption \ref{ass:lipschitz2}
holds for the online learning classification functions we consdier:
}
%
$\pi_{k'}(x,(t+1)/T) \leq \pi_{k'}(x,t/T) + L T^{-\alpha}$,
%
for some $L$ and $\alpha$ for all $k' \in {\cal F}$. Then we have the following theorem when online learning classifiers are present.
\begin{theorem}\label{theorem:cos2}
Let the CoS algorithm run with exploration control functions $D_1(t) = t^{2\alpha/(3\alpha+d+1)} \log t$, $D_2(t) = F_{\max} t^{2\alpha/(3\alpha+d+1)} \log t$, $D_3(t) = t^{2\alpha/(3\alpha+d+1)} \log t$ and slicing parameter $m_T = T^{1/(3\alpha + d+1)}$. Then, for any learner $i$, its regret is upper bounded by the following expression:
\add{\vspace{-0.1in}}
\begin{align*}
E[R_i(T)] &\leq T^{\frac{2\alpha+d+1}{3\alpha+d+1}}
\left( \frac{2 (2 L (d+1)^{\alpha/2}+6)}{\frac{2\alpha+d+1}{3\alpha+d+1}} + 2^{d+1} Z_i \log T \right) \\
&+ T^{\frac{\alpha+d+1}{3\alpha+d+1}} \frac{2^{d+3} (M-1) F_{\max} \beta_2}{2\alpha/(3\alpha+d+1)} \\
&+ T^{\frac{d}{3\alpha+d+1}} 2^{d+1} (2 Z_i \beta_2 
+ |{\cal K}_i|) + 4 (M-1) F_{\max} \beta_2,
\end{align*}
i.e., $R_i(T) = O \left(M F_{\max} T^{\frac{2\alpha+d+1}{3\alpha+d+1}} \right)$, where $Z_i = {\cal F}_i + (M-1)(F_{\max}+1)$.
\end{theorem} 
\begin{proof}
The proof is the same as proof of Theorem 1 in \cite{cem2013deccontext}, with context dimension $d+1$ instead of $d$.
\end{proof}

The above theorem implies that the regret in the presence of classification functions that learn online based on the data is $O(T^{(2\alpha+d+1)/(3\alpha+d+1)})$. \rev{From the result of Theorem \ref{theorem:cos2}, we see that our notion of context can capture any relevant information that can be utilized to improve the classification.
Specifically, we showed that by treating time as one dimension of the context, we can achieve sublinear regret bounds. Compared to Theorem 1 in \cite{cem2013deccontext}, in Theorem \ref{theorem:cos2}, the exploration rate is reduced from $O(T^{2\alpha/(2\alpha+d)})$ to $O(T^{2\alpha/(2\alpha+d+1)})$, while the memory requirement is increased from $O(T^{d/(3\alpha+d)})$ to $O(T^{(d+1)/(3\alpha+d+1)})$.}

\vspace{-0.1in}
\subsection{Delayed feedback}

Next, we consider the case when the feedback is delayed. We assume that the true label for data instance at time $t$ arrives to learner $i$ with an $L_i(t)$ time slot delay, where $L_i(t)$ is a random variable such that $L_i(t) \leq L_{\max}$ with probability one for some $L_{\max}>0$ which is known to the algorithm. Algorithm CoS is modified so that it keeps in its memory the last $L_{\max}$ labels produced by classification and the sample mean accuracies are updated whenever a true label arrives.
\newc{We assume that when a label arrives with a delay, the time slot of the incoming data stream which generated the label is known.}
We have the following result for delayed label feedback.
\begin{corollary} \label{cor:uniform}
Consider the delayed feedback case where the true label of the data instance at time $t$ arrives at time $t+L_i(t)$, where $L_i(t)$ is a random variable with support in $\{0,1,\ldots, L_{\max}\}$, $L_{\max}>0$ is an integer. Let $R_i^{\textrm{nd}}(T)$ denote the regret of CoS for learner $i$ with no delay by time $T$, and $R_i^{\textrm{d}}(T)$ denote the regret of modified CoS for learner $i$ with delay by time $T$. Then we have,
%
$R_i^{\textrm{d}}(T) \leq L_{\max} + R_i^{\textrm{nd}}(T)$.
%
\end{corollary}
\begin{proof}
\bremove{By a Chernoff-Hoeffding bound, it can be shown that the probability of deviation of the sample mean accuracy from the true accuracy decays exponentially with the number of samples. A new sample is added to sample mean accuracy whenever the true label of a previous classification arrives. Note that the worst case is when all labels are delayed by $L_{\max}$ time steps. This is equivalent to starting the algorithm with an $L_{\max}$ delay after initialization.} 
\badd{Due to the limited space the proof is given in our online appendix \cite{tekin2013arxivbig}.}
\end{proof}

The cost of label delay is additive which does not change the sublinear order of the regret. The memory requirement for CoS with no delay is $ |{\cal K}_i| (m_T)^d = 2^d (|{\cal F}_i| + M -1) T^{\frac{d}{ 3\alpha + d}}$, while memory requirement for CoS modified for delay is $L_{\max} + |{\cal K}_i| (m_T)^d$. Therefore, the order of memory cost is also independent of the delay. 
%

\vspace{-0.1in}
\newc{\subsection{True label is not always revealed}}

Sometimes it may not be possible to obtain the true label.
For example, the true label may not be revealed due to security reasons or failed communication. In this case it is not possible to update the sample mean rewards of the arms, therefore learning is interrupted. Assume that at each time step, the true label is revealed with probability $p_r$ (which can be unknown to the algorithm).
Let $M_i(t)$ be the number of times the true label is revealed to learner $i$ by time $t$. 
The following theorem gives an upper bound on the regret of CoS for this case. A similar regret bound can also be derived for DCZA.

\begin{theorem} \label{thm:nolabel}
Let the CoS algorithm run with exploration control functions $D_1(t) = t^{2\alpha/(3\alpha+d)} \log t$, $D_2(t) = F_{\max} t^{2\alpha/(3\alpha+d)} \log t$, $D_3(t) = t^{2\alpha/(3\alpha+d)} \log t$ and slicing parameter $m_T = T^{1/(3\alpha + d)}$. Then, for learner $i$,
\begin{align*}
R_i(T) &\leq T^{\frac{2\alpha+d}{3\alpha+d}}
\left( \frac{2 (2 L d^{\alpha/2}+6)}{(2\alpha+d)/(3\alpha+d)} + \frac{2^d Z_i}{p_r} \log T \right) \\
&+ T^{\frac{\alpha+d}{3\alpha+d}} \frac{2^{d+2} (M-1) F_{\max} \beta_2}{2\alpha/(3\alpha+d)} + + 4 (M-1) F_{\max} \beta_2\\
&+ T^{\frac{d}{3\alpha+d}} 2^d (2 Z_i \beta_2 
+ (|{\cal K}_i| + (M-1))/p_r),
\end{align*}
i.e., $R_i(T) = O \left(M F_{\max} T^{\frac{2\alpha+d}{3\alpha+d}}/p_r \right)$,
where $Z_i = |{\cal F}_i| + (M-1)(F_{\max}+1)$.
\end{theorem}
\begin{proof}
\bremove{Since the time slot $t$ is an exploitation slot only if $S_{i,l}(t) = \emptyset$ for $P_l$ which $x_i(t)$ belongs, the regret due to suboptimal and near optimal actions in exploitation steps will not be greater than the regret in the exploitations steps when the label is perfectly observed at each time step. Therefore the bounds given in Lemmas 2 and 4 in \cite{cem2013contextdata} will also hold for the case when label is not always observed. Only the regret due to explorations will be different, since more explorations are needed to observe sufficiently many labels such that $S_{i,l}(t) = \emptyset$.
Consider any partition $P_l$. From the definition of $S_{i,l}(T)$, the number of exploration steps in which a classification function $k_i \in {\cal F}_i$ is selected by learner $i$ and the label is observed is at most $\lceil T^{(2\alpha)/(3\alpha+d)}\rceil$, the number of training steps in which learner $i$ selects learner $j_i \in {\cal M}_{-i}$ and observes the true label is at most $\left\lceil F_{\max} T^{(2\alpha)/(3\alpha+d)} \log T \right\rceil$, and the number of exploration steps in which learner $i$ selects learner $j_i \in {\cal M}_{-i}$ is at most $\left\lceil T^{(2\alpha)/(3\alpha+d)} \log T \right\rceil$. 

Let $\tau_{\textrm{exp}}(T)$ be the random variable which denotes the smallest time step for which for each $k_i \in {\cal F}_i$ there are $\lceil T^{(2\alpha)/(3\alpha+d)}\rceil$ observations with label, for each $j_i \in {\cal M}_{-i}$ there are $\left\lceil F_{\max} T^{(2\alpha)/(3\alpha+d)} \log T \right\rceil$ observations with label for the trainings and $\left\lceil T^{(2\alpha)/(3\alpha+d)} \log T \right\rceil$ observations with label for the explorations. 
Then, $E[\tau_{\textrm{exp}}(T)]$ is the expected number of exploration slots by time $T$. Let $Y_{\textrm{exp}}(t)$ be the random variable which denotes the number of time slots in which the label is not revealed to learner $i$ till learner $i$ observes $t$ labels. Let $A_i(T) = Z_i T^{2\alpha/(3\alpha+d)} \log T +  (|{\cal F}_i| + 2(M-1))$. We have
%
$E[\tau_{\textrm{exp}}(T)] = E[Y_{\textrm{exp}}(A_i(T))] + A_i(T)$.
%
%
$Y_{\textrm{exp}}(A_i(T))$ is a negative binomial random variable with probability of observing no label at any time $t$ equals to $1-p_r$. Therefore $E[Y_{\textrm{exp}}(A_i(T))] = (1-p_r)A_i(T)/p_r$. Using this, we get
%
$E[\tau_{\textrm{exp}}(T)] = A_i(T)/p_r$.
%
The regret bound follows from substituting this into the proof of Theorem 1 in \cite{cem2013contextdata}.
}
\badd{Due to the limited space the proof is given in our online appendix \cite{tekin2013arxivbig}.}
\end{proof}

\subsection{Exploration reduction by increasing memory} \label{sec:reduction}

Whenever a new level $l$ hypercube is activated at time $t$, DCZA spends at least $O(t^z \log t)$ time steps to explore the arms in that hypercube. The actual number of explorations can be reduced by increasing the memory of DCZA. Each active level $l$ hypercube splits into $2^d$ level $l+1$ hypercubes when the number of arrivals to that hypercube exceeds $A 2^{pl}$. Let the level $l+1$ hypercubes formed by splitting of a level $l$ hypercube called {\em child} hypercubes. 
The idea is to keep $2^d$ sample mean estimates for each arm in ${\cal K}_i$ in each active level $l$ hypercube corresponding to its child level $l+1$ hypercubes, and to use the average of these sample means to exploit an arm when the level $l$ hypercube is active. Based on the arrival process to level $l$ hypercube, all level $l+1$ child hypercubes may have been explored more than $O(t^z \log t)$ times when they are activated. In the worst case, this guarantees that at least one level $l+1$ child hypercube is explored at least $A 2^{pl-d}$ times before being activated. The memory requirement of this modification is $2^d$ times the memory requirement of original DCZA, so in practice this modification is useful for $d$ small.

\vspace{-0.1in}
\section{Extensions for Distributed Stream Mining Problems} \label{sec:discuss}

In this section we describe several extensions to our online learning algorithms and provide some application areas, including how our framework can capture the concept drift, what happens when a learner only sends its context information to another learner, extensions to asynhornous and batch learning, choosing contexts adaptively over time, and extensions to networks of learners and ensemble learning.

\remove{
\begin{table}[t]
\centering
{\fontsize{8}{8}\selectfont
\setlength{\tabcolsep}{.1em}
\begin{tabular}{|l|c|c|c|c|c|}
\hline
& worst arrival  & worst arrival,   & best arrival, & best arrival  \\
&and correlation & best correlation & worst correlation& and correlation \\
\hline
CoS & $O \left(M F_{\max} T^{\frac{2\alpha+d}{3\alpha+d}} \right)$ 
& $O \left(|{\cal K}_i| T^{\frac{2\alpha+d}{3\alpha+d}} \right)$ 
& $O \left(M F_{\max}  T^{\frac{2\alpha}{3\alpha+d}} \right)$
& $O \left(|{\cal K}_i| T^{\frac{2\alpha}{3\alpha+d}} \right)$ \\
\hline
DDZA & $O \left(M F_{\max} T^{f_1(\alpha,d)}\right)$ 
& $O\left( |{\cal K}_i| T^{f_1(\alpha,d)}\right)$ 
& $O\left(M F_{\max} T^{2/3} \right)$ & $O\left(|{\cal K}_i| T^{2/3} \right)$ \\
\hline
\end{tabular}
}
\caption{Comparison CoS and DCZA}
\label{tab:compregret}
\end{table}
}
\comment{
\add{
\begin{table}[t]
\centering
{\fontsize{7}{6}\selectfont
\setlength{\tabcolsep}{1em}
\begin{tabular}{|l|c|c|}
\hline
& CoS & DDZA \\
\hline
worst arrival & $O \left(M F_{\max} T^{\frac{2\alpha+d}{3\alpha+d}} \right)$ & $O \left(M F_{\max} T^{f_1(\alpha,d)}\right)$  \\
and correlation& & \\
\hline
worst arrival, & $O \left(|{\cal K}_i| T^{\frac{2\alpha+d}{3\alpha+d}} \right)$ & $O\left( |{\cal K}_i| T^{f_1(\alpha,d)}\right)$  \\
best correlation & & \\
\hline
best arrival,& $O \left(M F_{\max}  T^{\frac{2\alpha}{3\alpha+d}} \right)$ & $O\left(M F_{\max} T^{2/3} \right)$ \\
worst correlation&& \\
\hline
best arrival & $O \left(|{\cal K}_i| T^{\frac{2\alpha}{3\alpha+d}} \right)$ & $O\left(|{\cal K}_i| T^{2/3} \right)$ \\
and correlation & & \\
\hline
\end{tabular}
}
\caption{Comparison CoS and DDZA}
\label{tab:compregret}
\vspace{-0.4in}
\end{table}
}
}
\vspace{-0.1in}
\subsection{Context to capture concept drift}


Formally, concept drift is a change in the distribution the problem \cite{gama2004learning, gao2007appropriate} over time. Examples of concept drift include recommender systems where the interests of users change over time and network security applications where the incoming and outgoing traffic patterns vary depending on the time of the day (see Section \ref{sec:numerical}).

Researchers have categorized concept drift according to the properties of the drift. Two important metrics are the {\em severity} and the {\em speed} of the drift given in \cite{minku2010impact}. The severity is the amount of changes that the new concept causes, while the {\em speed} of a drift is how fast the new concept takes place of the old concept. 
Both of these categories can be captured by our contextual data mining framework. Given a final time $T$, let $x_i(t) = t/T$ be the context for $i \in {\cal M}$. Thus $x_i(t) \in [0,1]$ always. Then the Lipschitz condition given in Assumption \ref{ass:lipschitz2} can be rewritten as
%
$|\pi_{k'}(t) - \pi_{k'}(t')| \leq (L|t-t'|^\alpha)/T^\alpha$.
%
Here $L$ captures the severity while $\alpha$ captures the speed of the drift. 
Our distributed learning algorithms CoS and DCZA can both be used to address concept drift, and provide sublinear convergence rate to the optimal classification scheme, given by the results of Theorems 1 and 2 in \cite{cem2013deccontext}, for $d=1$, by using time as the context information. 

Most of the previous work on concept drift focused on incremental and online ensemble learning techniques with the goal of characterizing the advantage of ensemble diversity under concept drift \cite{baena2006early, minku2012ddd, stanley2003learning, kolter2007dynamic}.
%
%
However, to the best of our knowledge all the previous methods are develop in an ad-hoc basis with no provable performance guarantees. In this subsection, we showed how our distributed contextual learning framework can be used to obtain regret bounds for classification under concept drift. Our learning framework can be extended to ensemble learning by jointly updating the sample mean accuracies of classification functions and the weights of the ensemble learner. \newc{We discuss more about this in Section \ref{sec:ensemble}, and provide numerical results comparing the performance of our online ensemble learning scheme with the related literature in Section \ref{sec:numerical}.}
%

%

\subsection{Sending only the context but not the data}

We note that for learner $i$ the communication cost of sending the data and receiving the prediction from another learner $j_i \in {\cal M}_{-i}$ is captured by the cost $d^i_{j_i}$. However, if $d^i_{j_i}$ is too high for $j_i \in {\cal M}_i$ compared to the costs of the classification functions in ${\cal F}_i$, then in the optimal distributed solution given in (\ref{eqn:opt2}) that requires full data exchange, learner $j_i$ may never be selected for any $x_i(t) \in {\cal X}$. In this case, algorithms CoS and DCZA will converge to the optimal solution that only uses the arms in ${\cal F}_i$. But is there a better way by which $i$ can exploit other good learners with smaller cost?
One solution is that instead of sending the high dimensional data, $i$ can send the low dimensional context to learner $j_i$. In this way the cost of communication will be much smaller than $d^i_{j_i}$ and may even be less than the costs $d^i_{k_i}$, $k_i \in {\cal F}_i$. Then, learner $j_i$ will not actually classify, but knowing the context, it will send back a prediction which has the highest percentage of being correct among all the predictions made by $j_i$ in the partition which the context belongs. This may outperform the optimal solution which requires full data exchange, especially if the prediction results of $j_i$ are strongly correlated with the context. Numerical results for this setting is given in Section \ref{sec:numerical}.
\newc{Also, if there are privacy concerns, sending only the context information is reasonable since this provides less information to the other learner $j_i$, the sending the data itself.}

\subsection{Cooperation among the learners}
In our analysis we assumed that learner $i$ can call any other learner $j_i \in {\cal M}_{-i}$ with a cost $d^i_{j_i}$, and $j_i \in {\cal M}_{-i}$ always sends back its prediction in the same time slot. However, learner $j_i$ also has to classify its own data stream thus it may not be possible for it to classify $i$'s data without delay. We considered the effect of a bounded delay in Corollary \ref{cor:uniform}. 
We note that there is also a cost for learner $j_i$ associated with communicating with learner $i$, but it is small since learner $j_i$ only needs to send $i$ its prediction but not the data as learner $i$ does. \newc{Even tough learner $j_i$ does not have an immediate benefit from classifying $i$'s data in terms of the reward, it has a long-term benefit from learning the result of the classification it performed for $i$, by updating its sample mean classification function accuracy.}
Similar to $i$, any other learner can use other learners to increase its prediction accuracy minus classification cost. Since the learners are cooperative, this
does not affect the optimal learning policy we derived for learner $i$.

%
\subsection{General reward functions}

In our analysis we assumed that the goal is to maximize the classification accuracy minus the cost which is captured by $\pi_k(x) - d^i_k$ for $k \in {\cal K}_i$, for learner $i$. Our setting can be extended to capture more general goals such as maximizing a function of accuracy and cost.
For example, consider a communication network $i$ with two arms $l$ and $k$, which are used to detect attacks, for which $d^i_k >> d^i_l$ but $0<\pi_k(x) - \pi_l(x)<<1$. Let $g_k(\pi_k(x),d^i_k)$ be the expected loss of arm $k$ given that context is $x$. 
The network can go down when attacked at a specific context $x'$, thus, the expected loss $g_l(\pi_l(x'),d^i_l)$ for arm $l$ can be much higher than the expected loss $g_k(\pi_k(x'),d^i_k)$ for arm $k$. Then, in the optimal solution, arm $k$ will be chosen instead of arm $l$ even though $d^i_k >> d^i_l$.
Our results for algorithms CoS and DCZA will hold for any general context dependent reward function $g_k(.)$, if Assumption \ref{ass:lipschitz2} holds for this reward function.

\newc{\subsection{Asynchronous and Batch Learning}}

In this paper, we assumed that at each time step a data stream with a specific context arrives to each learner. Although the number of arrivals is fixed, the arrival rate of data with different contexts is different for each learner because we made no assumptions on the context arrival process. However, we can generalize this setting to multiple data stream and context arrivals to each learner at each time instant. 
This can be viewed as data stream arriving to each learner in batches. Actions are taken for all the instances in the batch, and then the labels of all the instances are revealed only at the end of the time slot. CoS and DCZA can be modified so that the counters $N^i_{k,l}$ $N^i_{1,k,l}$ and $N^j_l$ are updated at the end of each time slot, based on the contexts in the batch for that time slot. Batch updating is similar to the case when the label is revealed with delay. Therefore, given that there are finite number of context and data arrivals to each learner at each time step, it can be shown that the upper bound on the regret for batch learning have the same time order with the original framework where a data stream with a single context arrives to each learner at each time slot. 

Another important remark is that both CoS and DCZA can be asynchronously implemented by the learners, since we require no correlation between the data and context arrivals to different learners. Learner $i$ selects an arm in ${\cal F}_i$ or ${\cal M}_{-i}$ only when a new data stream arrives to it, or even when there is no new data stream coming to learner $i$, it can keep learning by classifying the other learners data streams, when requested by these learners. 

\subsection{Unsupervised Learners}

So far we assumed that the each learner either instantly receives the label at the end of each time slot, or with a delay, or each learner receives the label with a positive probability. Another interesting setting is when some learners never receive the label for their data stream. Let $i$ be such a learner. The only way for $i$ to learn about the accuracies of arms ${\cal F}_i$, is to classify the data streams of the learners who receive labels. Since learner $i$ can only learn about accuracies when called by another learner who receives the label, in general it is not possible for learner $i$ to achieve sublinear regret. One interesting case is when the data/context arrival to learner $i$ is correlated with another learner $j$ who observes its own label at the end of each time slot. Consider the following modification of CoS for learner $i$. At each time step $t$, learner $i$ sends $x_i(t)$ to every other learner ${\cal M}_{-i}$. Based on $x_i(t)$, every learner $j$ sends back to $i$ the sample mean accuracy of their estimated best classification function for $P_l$ such that $x_i(t) \in P_l$. Then, to classify its data stream learner $i$ selects the arm in ${\cal K}_i$ with the highest expected accuracy.   
If the correlation is such that whenever $x_{j'}(t) \in P_l$ we have $x_i(t) \in P_l$ for some $P_l \in {\cal P}_T$ for all $j,j' \in {\cal M}$ , then the regret of CoS for learner $i$ will be the same as Theorem 1 in \cite{cem2013deccontext}, since the trainings and explorations of learners who receive the label 
will be enough for learners who do not receive any label to estimate the accuracies of their own classification functions correctly with a high probability. 

On the contrary, even for simple cases such as independent data/context arrivals to each learner, an unsupervised learner $i$ may not achieve sublinear regret. We illustrate this in the following example.

\begin{example-non*}
Let $i$ be an unsupervised learner. Let ${\cal P}^i_T$ be the sets in ${\cal P}_T$ in which there exists at least one $x_i(t)$, $t \leq T$ and ${\cal P}^{-i}_T$ be the sets in ${\cal P}_T$ in which there exists at least one $(x_j(t))_{j \in {\cal M}_{-i}}$, $t \leq T$. For stochastic context arrivals ${\cal P}^i_T$ and ${\cal P}^{-i}_T$ are random variables. If $P({\cal P}^i_T \cap {\cal P}^j_T =  \emptyset ) > 0 $, then it is not possible for learner $i$ to achieve sublinear regret. This is because with positive probability, learner $i$ will learn nothing about the accuracy of its own classification functions for its context realization $x_i(1), \ldots, x_i(T)$. This means that it cannot do better than random guessing with positive probability, hence the regret will be linear in $T$.
\end{example-non*}

%
\vspace{-0.2in}
\newc{\subsection{Choosing contexts adaptively over time}}

We discussed that context can be one or multiple of many things such as the time, location, ID, or some other features of the incoming data stream. Given what we take the set ${\cal X}$ to be, the classification accuracies $\pi_k(x)$ will change. Since the time order of the regret grows exponentially with the dimension of the context space, sometimes it might be better to consider only a single feature of the incoming data stream as context. Assume that the number of features that can be used as context is $d$. At time $t$, $\boldsymbol{x}_i(t) = (x^1_i(t), x^2_i(t), \ldots, x^d_i(t))$ arrives to learner $i$ where $x^m_i(t) \in (0,1]$ for $m=1,\ldots,d$.
CoS (also DCZA) can be modified in the following way to adaptively choose the best context which maximizes the expected classification accuracy. We call the modified algorithm {\em CoS with multiple contexts} (CoS-MC).
Let 
%
${\cal S}_{i}(\boldsymbol{x}_i(t),t) = \cup_{x^m_i(t) \in \boldsymbol{x}_i(t)} {\cal S}^m_{i}(x^m_i(t),t)$,
%
where
%
${\cal S}^m_{i}(x^m_i(t),t) = {\cal S}^m_{i,l}(t)$ for $x^m_i(t) \in \left(\frac{l-1}{m_T},\frac{l}{m_T}\right]$,
$l=1,2,\ldots,m_T$,
%
and
\begin{align*}
{\cal S}^m_{i,l}(t) &= \{ k_i \in {\cal F}_i : N^{i,m}_{k_i,l}(t) \leq D_1(t) \textrm{ or } j_i \in {\cal M}_{-i} : \\
& N^{i,m}_{1,j_i,l}(t)  \leq D_2(t) \textrm{ or } N^{i,m}_{j_i,l}(t) \leq D_3(t)  \},
\end{align*}
where similar to the counters of standard CoS algorithm, $N^{i,m}_{k_i,l}(t)$, $N^{i,m}_{1,j_i,l}(t)$ and $N^{i,m}_{j_i,l}(t)$ represents the number of explorations of classification function $k_i \in {\cal K}_i$, trainings of learner $j_i \in {\cal M}_{-i}$ and explorations of learner $j_i \in {\cal M}_{-i}$ at the time steps the context lies in $P_l$ by time $t$, respectively.

At time $t$, CoS-MC randomly selects an arm in set ${\cal S}_{i}(\boldsymbol{x}_i(t),t)$ if ${\cal S}_{i}(\boldsymbol{x}_i(t),t) \neq \emptyset$. Otherwise, it selects the arm which offers the maximum estimated reward among all contexts, i.e.,
\begin{align}
\argmax_{k \in {\cal K}_i, m=1,\ldots,d} \bar{r}^{i,m}_{k,l(x^m_i(t))}(t), \label{eqn:multicontext}
\end{align}
where $\bar{r}^{i,m}_{k,l}(t)$ is the sample mean of the rewards collected from times when $m$th context is in $P_l$ and arm $k$ is chosen by time $t$. Note that independent of the specific context which arm selection at time $t$ depends on, the sample mean rewards corresponding to the selected arm $k$ of all sets to which the contexts at time $t$ belongs are updated based on the comparison of the prediction and the label. 

Let $\pi^m_k(x^m)$ be the expected accuracy of arm $k \in {\cal K}_i$ given $m$th context $x^m$. We assume that Assumption \ref{ass:lipschitz2} holds for all $\pi^m_k(x^m)$, $m=1,\ldots,d$ for some constants $\alpha>0$ and $L>0$.
For $j_i \in {\cal M}_{-i}$, let $\pi^m_{j_i}(x) := \max_{k_{j_i} \in {\cal F}_{j_i}} \pi^m_{k_{j_i}}(x)$.
We define the best single-context policy given $\pi^m_{k'}(x)$ for all $k' \in {\cal F}$, $m=1,\ldots,d$ as
%
$k^{\textrm{sb}}_i(\boldsymbol{x}) := \argmax_{k \in {\cal K}_i, m=1,\ldots,d} \pi^m_{k}(x^m) - d^i_k$.
%
Let $\pi^{\textrm{sb}}_k (\boldsymbol{x}) := \max_{m=1,\ldots,d} \pi^m_{k}(x^m) - d^i_k$, for $k \in {\cal K}_i$. The regret CoS-MC with respect to the best single-context policy is given by
\begin{align*}
&R^{\textrm{sb}}_i(T) := \sum_{t=1}^T  \pi^{\textrm{sb}}_{k^{\textrm{sb}}_i(\boldsymbol{x}_i(t))}(\boldsymbol{x}_i(t)) \\ 
&- E \left[ \sum_{t=1}^T ( I(\hat{y}_t(\alpha^{\textrm{CoS-CM}}_t(\boldsymbol{x}_i(t))) = y^i_t) - d^i_{\alpha^{\textrm{CoS-CM}}_t(\boldsymbol{x}_i(t))}) \right] ,
\end{align*}
Note that the best-single context policy is always better than the policy which a-priori selects one of the $d$ contexts (e.g., $m$th context) as its context, and selects arms optimally at each time step based on the context $x^m_i(t)$ only. In general, the best-single context policy is worse than the best policy which uses all $d$ contexts together to select an arm.

\begin{theorem}
Let the CoS-MC algorithm run with exploration control functions $D_1(t) = t^{2\alpha/(3\alpha+2)} \log t$, $D_2(t) = F_{\max} t^{2\alpha/(3\alpha+2)} \log t$, $D_3(t) = t^{2\alpha/(3\alpha+2)} \log t$ and slicing parameter $m_T = T^{1/(3\alpha + 2)}$. Then, for any learner $i$, we have
%
$R^{\textrm{sb}}_i(T) = O\left( d M F_{\max} T^{\frac{2\alpha+1}{3\alpha+1}}   \right)$.
%
\end{theorem}
\begin{IEEEproof}
\bremove{It is easy to see that the number of explorations and trainings of CoS-MC is at most $d$ times the number of explorations and trainings of CoS run with context dimension equal to 1. In the exploitation steps, in order to be optimal or near-optimal the arm chosen in (\ref{eqn:multicontext}) requires maximization over $d {\cal K}_i$ possible values, therefore, the regret due to suboptimal and near-optimal arm selections in CoS-MC is at most $d$ times the number of suboptimal and near-optimal arm selections in CoS run with context dimension equal to 1. A finite-time regret bound can also be proved by following steps similar to the proof of Theorem 1 in \cite{cem2013deccontext}.}
\badd{Due to the limited space the proof is given in our online appendix \cite{tekin2013arxivbig}.}
\end{IEEEproof}

\subsection{Instance distributed vs. feature distributed}

In our formulation, the incoming data stream of each learner can either be instance (horizontally) or feature (vertically) distributed. For feature distributed data, context may give information about what features to extract from the data. Note that if the features arriving to each learner is different from the features of other learner, then the context arrival process is the same as {\em worst-case correlation} described in Definition 1 in \cite{cem2013deccontext}. Basically, context space of learner $i$ is different form the context spaces of other learners, therefore this is equivalent to the case where no data/context arrives to the other learners from the perspective of learner $i$. The accuracies of other learners for $i$'s data/context is only learned by the trainings and exploration of the other learners by $i$. Therefore Theorems 1 and 2 in \cite{cem2013deccontext} for {\em worst-case correlation} holds for any learner $i$ for feature distributed data.

\newc{\subsection{Extension to online ensemble learning}} \label{sec:ensemble}

In this paper we assumed that the goal of each learner is to maximize the expected number of correct predictions about its own data stream, based on the prediction of a single classification function which is either one of its own classification functions or another learner's classification function. Another interesting online classification problem, which is studied by many researchers \cite{shalev2011pegasos, littlestone1989weighted, littlestone1988learning}, is to combine the predictions of individual learners to generate an ensemble prediction which is usually more accurate then the predictions of individual learners. 

\begin{figure}
\begin{center}
\includegraphics[width=0.9\columnwidth]{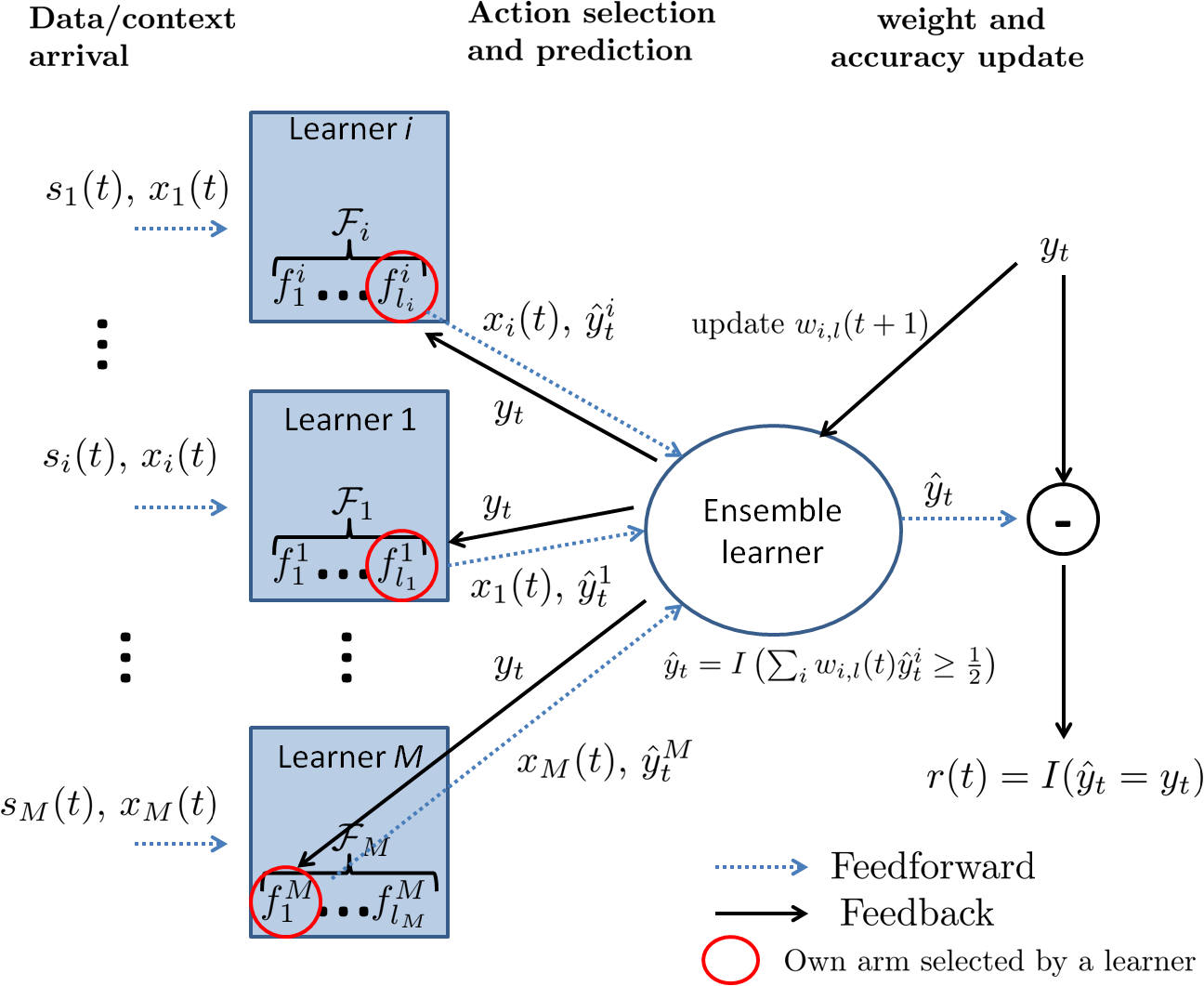}
\vspace{-0.1in}
\caption{Ensemble learning framework, where each learner only selects one of its own arms, and sends its prediction to the ensemble learner which produces a final prediction.} 
\label{fig:ensemble}
\end{center}
\vspace{-0.2in}
\end{figure}

Consider the system model given in Fig. \ref{fig:ensemble}. At each time step $t$, an instance with context $x_i(t)$ arrives to local learner $i$. Different from the previous sections, assume that learner $i$ chooses one of its classification functions in ${\cal F}_i$ (but not ${\cal K}_i$) based on its context and produces a prediction $\hat{y}^i_t$. Then, each learner $i$ sends its prediction and context to the ensemble learner. The ensemble learner checks to which set $P_l \in {\cal P}_T$ the context $x_i(t)$ belongs, and assigns weight $w_{i,l}(t)$ to learner $i$ which depends on history of $i$'s predictions in set $P_l$ by time $t$.
Then, the final prediction is made by a weighted majority rule, i.e.,
%
$\hat{y}_t = 1 \textrm{ if } \sum_{i \in {\cal M}} w_{i,l}(t) \hat{y}^i_t \geq 1/2$, and 
$\hat{y}_t = 0$ otherwise.
%
At the end of time slot $t$, the label is revealed to both the ensemble learner and the local learners. The goal is to maximize the expected number of correct predictions made by the ensemble learner, i.e., $E[\sum_{t=1}^T I(\hat{y}_t = y_t)]$.

Comparing the result of their predictions with the label, local learners update the estimated accuracies of their chosen classification functions, while the ensemble learner updates the weights of the local learners. 
\newc{For each set $P_l \in {\cal P}_T$, the weights can be updated using stochastic gradient descent methods \cite{shalev2011pegasos} or the weights corresponding to learners with false predictions can be decreased and the learners with correct predictions can be increased multiplicatively similar to the weighted majority algorithm and its variants \cite{littlestone1989weighted, littlestone1988learning}.
However, although some of these weight update methods are shown to asymptotically converge to the optimal weight vector, it is not possible to obtain finite-time regret bounds for these methods. It is an interesting future research direction to develop online learning methods for updating weights which will give sublinear regret bounds for the ensemble learner. Numerical results related to the ensemble learner is given in Section \ref{sec:numerical}.
} 

\vspace{-0.1in}
\subsection{Distributed online learning for learners in a network}

%
In general, learners may be distributed over a network, and direct connections may not exist between learners. For example consider the network in Fig. \ref{fig:network}. Here, learner $i$ cannot communicate with learner $j$ but there is a path which connects learner $i$ to learner $j$ via learner $j'$ or $j''$. We assume that every learner knows the network topology and the lowest-cost paths to every other learner.
Our online learning framework can be directly applied in this case. Indeed, this is a special case of our framework in which the cost $d^i_j \geq d^i_{j'}$. For example, if the cost is delay, then we have
$d^i_{j} = d^i_{j'} + d^{j'}_j$. When the cost is delay cost, in general for learner $i$, the cost of choosing learner $j$ is equal to the sum of the costs among the lowest-cost path between learner $i$ and $j$. Note that similar to the previous analysis, we assume that the lowest-cost path costs are normalized to be in $[0,1]$. If the lowest-cost path between two learner $i$ and $j$ is greater than 1, independent of the classification accuracy of learner $j$, learner $i$'s reward of choosing leaner $j$ will be negative, which means that learner $i$ will never call learner $j$. 

\begin{figure}
\begin{center}
\includegraphics[width=0.6\columnwidth]{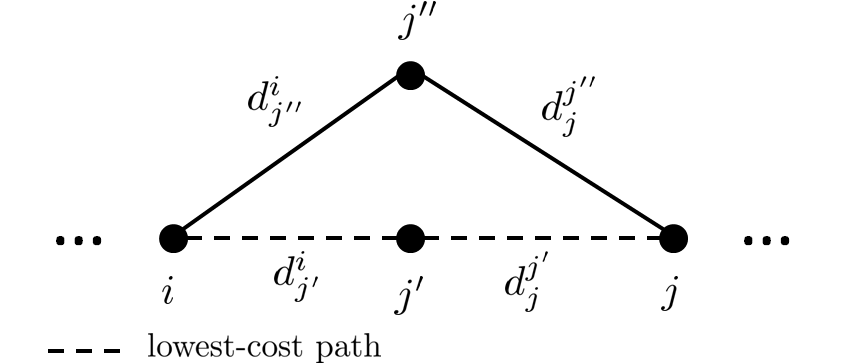}
\vspace{-0.1in}
\caption{A network topology in which learner $i$ have no direct connection with learner $j$.} 
\label{fig:network}
\end{center}
\vspace{-0.2in}
\end{figure}

This network scenario can be generalized such that the link costs between the learners can be unknown and time-varying, or the topology of the network may be unknown and time-varying. We leave the investigation of these interesting scenarios as a future work.

\section{Numerical Results} \label{sec:numerical}

In this section we provide numerical results for our proposed algorithms CoS and DCZA both using a real-world data set. In the following definition we give different context arrival processes which captures the four extreme points of context arrivals. 

\begin{definition}\label{defn:context}
We call the context arrival process $\{(x_1(t), \ldots, x_K(t))\}_{t=1,\ldots,T}$, {\em the worst-case arrival process} if for each $i \in {\cal M}$, $\{x_i(t)\}_{t=1,\ldots,T}$ is uniformly distributed inside the context space, with minimum distance between any two context samples being $T^{-1/d}$; {\em the best-case arrival process} if for each $i \in {\cal M}$, $x_i(t) \in C$ for all $t=1,\ldots,T$ for some level $\lceil(\log_2 T)/p\rceil +1$ hypercube $C$.
We say the context arrival process has {\em worst-case correlation} if context only arrives to learner $i$ (no context arrivals to other learners); has {\em best-case correlation} if $x_i(t) = x_j(t)$ for all $i,j \in {\cal M}$, $t=1,\ldots,T$. We define the following four cases to capture the extreme points of operation of DCZA:
\begin{itemize}
\item \textbf{C1} worst-case arrival and correlation
\item \textbf{C2} worst-case arrival, best-case correlation
\item \textbf{C3} best-case arrival, worst-case correlation
\item \textbf{C4} best-case arrival and correlation
\end{itemize}
\end{definition}

\vspace{-0.1in}
\subsection{Simulation Setup}

For our simulations, we use the network security data from KDD Cup 1999 data set. We compare the performance of our learning algorithms with {\em AdaBoost} \cite{freund1995desicion} and the online version of AdaBoost called {\em sliding window AdaBoost} \cite{oza2001online}. 

The network security data has 42 features. The goal is to predict at any given time if an attack occurs or not based on the values of the features.
We run the simulations for three different context information; (A1) context is the label at the previous time step, (A2) context is the feature named {\em srcbytes}, which is the number of data bytes from source to destination, (A3) context is time. All the context information is normalized to be in $[0,1]$. There are $4$ local learners. Each local learner has $2$ classification functions. \newc{Unless noted otherwise}, the classification costs $d_k$ are set to $0$ for all $k \in {\cal K}_1$.

All classification functions are trained using 5000 consecutive samples from different segments of the network security data. Then, they are tested on $T = 20000$ consecutive samples. We run simulations for two different sets of classifiers. 
In our first simulation S1, there are two good classifiers that have low number of errors on the test data, while in our second simulation S2, there are no good classifiers. The types of classification functions used in S1 and S2 are given in Table \ref{tab:sim_setup} along with the number of errors each of these classification functions made on the test data. From Table \ref{tab:sim_setup}, we can observe that the error percentage of the best classification function is $3$ in S1, while it is $47$ in S2. A situation like in S2 can appear when the distribution of the data changes abruptly, i.e., concept drift, so that the classification functions trained on the old data becomes inaccurate for the new data. In our numerical results, we will show how the context information can be used to improve the performance in both S1 and S2.
The accuracies of the classifiers on the test data are unknown to the learners so they cannot simply choose the best classification function. 
In all our simulations, we assume that the test data sequentially arrives to the system and the label is revealed to the algorithms with a one step delay. 

\begin{table}[t]
\centering
{\fontsize{8}{6}\selectfont
\setlength{\tabcolsep}{.1em}
\begin{tabular}{|l|c|c|c|c|c|}
\hline
Learner & 1 &2 &3 & 4\\
\hline
Classification  & Naive Bayes,  &Always $1$, & RBF Network, & Random Tree,  \\
Function (S1)  & Logistic & Voted Perceptron &J48 & Always $0$\\
\hline
 Error  & 47,  & 53,  & 47, & 47,  \\
percentage (S1) & 3 & 4 &  47 & 47\\
\hline
 Classification  & Naive Bayes,  &Always $1$,  & RBF Network,  & Random Tree,  \\
Function (S2) & Random & Random & J48 & Always $0$ \\
\hline
 Error & 47,  & 53,  & 47,  & 47,  \\ 
percentage (S2) & 50 & 50 & 47 & 47 \\
\hline
\end{tabular}
}
\add{\vspace{-0.1in}}
\caption{Base classification functions used by the learners and their error percentages on the test data.}
\vspace{-0.25in}
\label{tab:sim_setup}
\end{table}

Since we only consider single dimensional context, $d=1$. However, due to the bursty, non-stochastic nature of the network security data we cannot find a value $\alpha$ for which Assumption \ref{ass:lipschitz2} is true. Nevertheless, we consider two cases, Z1 and Z2, given in Table \ref{tab:par_setup}, for CoS and DCZA parameter values.
In Z2, the parameters for CoS and DCZA are selected according to Theorems 1 and 2 in \cite{cem2013deccontext}, assuming $\alpha=1$. 
In Z1, the parameter values are selected in a way that will reduce the number of explorations and trainings. However, the regret bounds for Theorems 1 and 2 in \cite{cem2013deccontext} may not hold for these values in general. 

\begin{table}[t]
\centering
{\fontsize{8}{6}\selectfont
\setlength{\tabcolsep}{.3em}
\begin{tabular}{|l|c|c|c|c|c|c|}
\hline
 & $D_1(t)$ & $D_2(t)$ & $D_3(t)$ & $m_T$ & $A$ & $p$ \\
\hline
(Z1) CoS & $t^{1/8} \log t$ & $2 t^{1/8} \log t$ & $t^{1/8} \log t$ & $\lceil T \rceil^{1/4}$ & & \\
\hline
(Z1) DCZA & $t^{1/8} \log t$ & $2 t^{1/8} \log t$ & $t^{1/8} \log t$ &  & $1$ & $4$ \\
\hline
(Z2) CoS & $t^{1/2} \log t$ & $2 t^{1/2} \log t$ & $t^{1/2} \log t$ & $\lceil T \rceil^{1/4}$ & & \\
\hline
(Z2) DCZA & $t^{2/p} \log t$ & $2 t^{2/p} \log t$ & $t^{2/p} \log t$ &  & $1$ & $(3+\sqrt{17})/2$ \\
\hline
\end{tabular}
}
\add{\vspace{-0.05in}}
\caption{Input parameters for CoS and DCZA for two different parameter sets Z1 and Z2.}
\label{tab:par_setup}
\add{\vspace{-0.4in}}
\end{table}

\vspace{-0.1in}
\subsection{Simulation Results for CoS and DCZA}

In our simulations we consider the performance of learner 1. Table \ref{tab:sim_results} shows under each simulation and parameter setup the percentage of errors made by CoS and DCZA and the percentage of time steps spent in training and exploration phases for learner 1. We compare the performance of DCZA and CoS with AdaBoost, sliding window AdaBoost (SWA), and CoS with no context (but still decentralized different from a standard bandit algorithm) whose error rates are also given in Table \ref{tab:error_comp}. AdaBoost and SWA are trained using 20000 consecutive samples from the data set different from the test data. 
SWA re-trains itself in an online way using the last $w$ observations, which is called the window length. Both AdaBoost and SWA are ensemble learning methods which require learner 1 to combine the predictions of all the classification functions. Therefore, when implementing these algorithms we assume that learner 1 has access to all classification functions and their predictions, whereas when using our algorithms we assume that learner 1 only has access to its own classification functions and other learners but not their classification functions. Moreover, learner 1 is limited to use a single prediction in CoS and DCZA. This may be the case in a real system when the computational capability of local learners are limited and the communication costs are high.

First, we consider the case when the parameter values are as given in Z1. We observe that when the context is the previous label, CoS and DCZA perform better than AdaBoost and SWA for both S1 and S2. This result shows that although CoS and DCZA only use the prediction of a single classification function, by exploiting the context information they can perform better than ensemble learning approaches which combine the predictions of all classification functions. We see that the error percentage is smallest for CoS and DCZA when the context is the previous label.
This is due to the bursty nature of the attacks.
The exploration percentage for the case when context is the previous label is larger for DCZA than CoS. 
As we discussed in Section \ref{sec:reduction}, the number of explorations of DCZA can be reduced by utilizing the observations from the old hypercube to learn about the accuracy of the arms in a newly activated hypercube. When the context is the feature of the data or the time, for S1, CoS and DCZA perform better than AdaBoost while SWA with window length $w=100$ can be slightly better than CoS and DCZA. But again, this difference is not due to the fact that CoS and DCZA makes too many errors. It is because of the fact that CoS and DCZA explores and trains other classification functions and learners. 
AdaBoost and SWA does not require these phases. But they require communication of predictions of all classification functions and communication of all local learners with each other at each time step. Moreover, SWA re-trains itself by using the predictions and labels in its time window, which makes it computationally inefficient.
Another observation is that using the feature as context is not very efficient when there are no good classifiers (S2). However, the error percentages of CoS and DCZA ($39\%$ and $38\%$ respectively) are still lower than the error percentage of the best classifier in S2 which is $47\%$. Moreover, CoS and DCZA performs better than CoS with no context for all scenarios with parameter values given by Z1. 
We observe that both CoS and DCZA performs poorly when the set of parameters is given by Z2. This is due to the fact that the percentage of training and exploration phases is too large for Z2, thus these algorithms cannot exploit the information they gathered efficiently. Another important reason for the poor performance is the short time horizon. As the time horizon grows, we expect the exploration and training rates to decrease, and the exploitation rate to increase which will improve the performance. 

\begin{table}[t]
\centering
{\fontsize{8}{6}\selectfont
\setlength{\tabcolsep}{.3em}
\begin{tabular}{|l|c|c|}
\hline
(Parameters) Algorithm & (S1) Error $\%$  & (S2) Error $\%$ \\
\hline
(Z1) CoS (previous label as context) & 0.7 & 0.9  \\
\hline
(Z1) DCZA (previous label as context) & 1.4 & 1.9  \\
\hline
AdaBoost & 4.8 & 53  \\
\hline
($w=100$) SWA  & 2.4 & 2.7 \\
\hline
($w=1000$) SWA  & 11 & 11  \\
\hline
(Z1) CoS (no-context) & 5.2 & 49.8  \\
\hline 
\end{tabular}
}
\add{\vspace{-0.05in}}
\caption{Comparison of error percentages of CoS, DCZA, AdaBoost, SWA and CoS with no context.}
\label{tab:error_comp}
\vspace{-0.2in}
\end{table}

\begin{table}[t]
\centering
{\fontsize{8}{6}\selectfont
\setlength{\tabcolsep}{.3em}
\begin{tabular}{|l|c|c|c|}
\hline
(Setting) & Error $\%$ & Training $\%$ & Exploration $\%$  \\
Algorithm  & context=A1,A2,A3 & context=A1,A2,A3 & context=A1,A2,A3 \\
\hline
(Z1,S1) CoS & 0.7, 4.6, 4.8 & 0.3, 3, 2.8 & 1.4, 6.3, 8.5 \\
\hline
(Z1,S1) DCZA & 1.4, 3.5, 3.2 & 0.4, 1.3, 0.9 & 4, 5.9, 7 \\
\hline
(Z1,S2) CoS & 0.9, 39, 10 & 0.3, 3, 2.8 & 1.5, 6.5, 8.6 \\
\hline
(Z1,S2) DCZA & 1.9, 38, 4.8 & 0.4, 1.3, 1 & 4, 6, 7 \\
\hline
(Z2,S1) CoS & 16, 14, 41 & 8.5, 16, 79 & 55 27 20\\
\hline
(Z2,S1) DCZA & 31, 29, 29 & 33 19 87 & 66 66 12 \\
\hline 
\end{tabular}
}
\add{\vspace{-0.05in}}
\caption{Error, training and exploration percentages of CoS and DCZA under different simulation and parameter settings. (A1) context as the previous label, (A2) context as srcbytes feature, (A3) context as time.}
\label{tab:sim_results}
\vspace{-0.2in}
\end{table}

\newc{The results in Table \ref{tab:sim_results} are derived for the case when all learners receive the same instance and observe the same label at each time step. Therefore they correspond to the {\em best-case correlation} given in Definition \ref{defn:context}. Moreover, when the context is time, we have {\em worst-case} arrival (as in \textbf{C2}), while when the context is the previous label, arrivals are similar to the {\em best-case arrival} process (as in \textbf{C4}), where instead of a single hypercube of level $\lceil(\log_2 T)/p\rceil +1$, arrivals happen to two different level $\lceil(\log_2 T)/p\rceil +1$ hypercubes one containing $x=0$ and the other one containing $x=1$. 
We also run simulations for the worst-case correlation (\textbf{C1} or \textbf{C3}) for CoS (results for DCZA will be similar) for three different contexts A1, A2 and A3. From the results given in Table \ref{tab:sim_results2}, we observe that the exploration and training percentages increases for the worst-case correlation between the learners, which also causes an increase in the error percentages. 
}

\begin{table}[t]
\centering
{\fontsize{8}{6}\selectfont
\setlength{\tabcolsep}{.3em}
\begin{tabular}{|l|c|c|c|}
\hline
(Setting) & Error $\%$ & Training $\%$ & Exploration $\%$  \\
Algorithm & context=A1,A2,A3 & context=A1,A2,A3 & context=A1,A2,A3 \\
\hline
(Z1,S1) CoS & 1.8, 4.1, 6.7 & 2, 9.2, 10.3 & 1.4, 3.6, 8.5 \\
\hline
(Z1,S2) CoS & 24.6, 44.3, 31.3 & 2, 9.2, 10.3 & 1.4, 3.6, 8.5  \\
\hline
\end{tabular}
}
\caption{Error, training and exploration percentages of CoS for  worst-case correlation between the learners for three different context types.}
\label{tab:sim_results2}
\vspace{-0.2in}
\end{table}

\vspace{-0.1in}
\subsection{Simulation Results for Extensions on CoS and DCZA}

\newc{Firstly, we simulate the ensemble learner given in Section \ref{sec:ensemble} for CoS (called {\em ensemble CoS}), with $d_k=0$ for $k \in {\cal K}_i$, $i \in {\cal M}$. We take time as the context, and consider two different weight update rules. In the {\em context-independent} update rule, weights $w_i(1)$ for each learner is initially set to $1/4$, and $\boldsymbol{w}(t) = (w_1(t),\ldots, w_4(t))$ is updated based on the stochastic gradient descent rule given in Algorithm 2 of \cite{yu2013fast}, with coefficient $1/\alpha$ instead of $1/(\alpha*t)$ to capture the non-stationarity of the incoming data stream where $\alpha=100$. In the {\em context-dependent} update rule, weights for each learner in each set in the partition ${\cal P}_T$ is updated independently from the weights in the other sets based on the same stochastic gradient descent rule.
Total error and exploitation error percentages of ensemble CoS is given in Table \ref{tab:weights} for cases S1 and S2. Comparing Tables \ref{tab:sim_results} and \ref{tab:weights}, we see that when the weight update rule is context-independent, there is $21\%$ and $51\%$ improvement in the error of ensemble CoS compared to CoS for cases S1 and S2 respectively. However, when the weight update rule is context dependent, ensemble CoS performs worse than CoS. This is due to the fact that the convergence rate being smaller for context-dependent weights since weights for each $P_l \in {\cal P}_T$ are updated independently. In Table \ref{tab:weights}, we also give the percentage of prediction errors made at the time slots in which all learners are simultaneously in the exploitation phase of CoS. The difference between total error percentage and exploitation error percentage gives the percentage of errors made in exploration steps.
\begin{table}[t]
\centering
{\fontsize{8}{6}\selectfont
\setlength{\tabcolsep}{.3em}
\begin{tabular}{|l|c|c|}
\hline
ensemble CoS & context-dependent weights & context-indep weights \\
Parameters: Z1 & S1, S2 & S1, S2 \\
\hline
total error $\%$ & 5.9, 10.2 & 3.8, 4.94  \\
\hline
exploitation error $\%$ & 2.9, 6.8 & 1.76, 2.17  \\
\hline
\end{tabular}
}
\caption{Total error percentage, and error percentage of the errors made in exploitation steps for CoS with ensemble learner.}
\label{tab:weights}
\vspace{-0.2in}
\end{table}
}
\comment{
In Figure \ref{fig:CoS_res} and \ref{fig:DCZA_res}, the number of classification errors is given as a function of time for CoS and DCZA, respectively for S1 and Z1, for all three different context information. The error function is {\em spike-shaped} because of the training phases that happens when a change in the context occurs. Note that when the context is previous label or {\em srcbytes}, the spikes follow the changes in the data, while when the context is time, the spikes are uniformly distributed. 

\begin{figure}
\begin{center}
\includegraphics[width=0.8\columnwidth]{CoS_res}
\caption{Number of errors made by CoS as a function of time.} 
\label{fig:CoS_res}
\end{center}
\end{figure}

\begin{figure}
\begin{center}
\includegraphics[width=0.8\columnwidth]{DCZA_res}
\caption{Number of errors made by DCZA as a function of time.} 
\label{fig:DCZA_res}
\end{center}
\end{figure}
}

Secondly, we simulate both CoS and DCZA for the case when the label is not always observed. Our results are given in Table \ref{tab:errorperc} for learner 1 when there are 4 learners. As the probability of observing the label, i.e., $p_r$, decreases, the error percentage of both CoS and DCZA grows. This is due to the fact that more time steps are spent in exploration and training phases to ensure that the estimates of the rewards of arms in ${\cal K}_i$ are accurate enough. 

\comment{
\begin{table}[t]
\centering
{\fontsize{8}{6}\selectfont
\setlength{\tabcolsep}{.1em}
\begin{tabular}{|l|c|c|c|c|c|}
\hline
Learner & 1 &2 &3 & 4\\
\hline
Classification  & Naive Bayes,  &Multilayer Perceptron, & RBF Network, & Random Tree,  \\
Function   & Logistic & Voted Perceptron &J48 & Random Forest\\
\hline
 Error  & 42.9,  & 33,  & 43, & 27.5,  \\
percentage & 42.9 & 42.9 &  42.9 & 22\\
\hline
\end{tabular}
}
\add{\vspace{-0.1in}}
\caption{Error percentages of the base classification functions for synthetic data}
\vspace{-0.25in}
\label{tab:basesynt}
\end{table}
}

\comment{
\begin{figure}
\begin{center}
\includegraphics[width=0.8\columnwidth]{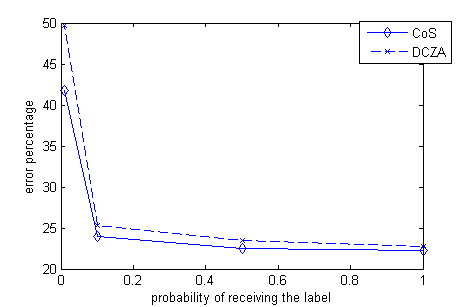}
\caption{Error percentage as a function of the probability of observing the label for the instance.} 
\label{fig:somelabel}
\end{center}
\end{figure}
}

Thirdly, in Table \ref{tab:nlearn}, error percentages of CoS and DCZA for learner 1 are given as functions of the number of learners present in the system. Comparing Tables \ref{tab:sim_setup} and \ref{tab:nlearn}, we see that adding a learner whose classification functions are worse than the classification functions of the current learners increases the error percentage due to the increase in the number of training and exploration steps, while adding a learner whose classification functions are better than the best classification function of the current learners decreases the error percentage.

\begin{table}[t]
\centering
{\fontsize{8}{6}\selectfont
\setlength{\tabcolsep}{.1em}
\begin{tabular}{|l|c|c|c|c|c|}
\hline
(Setting) Algorithm /$p_r$ & 1 & 0.5 & 0.1 & 0.01 \\
\hline
(Z1,S2) CoS (context is time) error $\%$ & 10 & 13.9 & 36.4 & 47.1 \\
\hline
(Z1,S2) DCZA (context is time) error $\%$ & 4.8 & 4.8 & 16.3 & 56.6 \\
\hline
\end{tabular}
}
\caption{Error percentages of CoS and DCZA as a function of $p_r$ (probability of receiving the label at each time slot) when context is time.}
\vspace{-0.2in}
\label{tab:errorperc}
\end{table}

\begin{table}[t]
\centering
{\fontsize{8}{6}\selectfont
\setlength{\tabcolsep}{.1em}
\begin{tabular}{|l|c|c|c|c|c|}
\hline
$\#$ of learners & 1 & 2 & 3 & 4 \\
\hline
CoS error $\%$ & 49.8 & 49.7 & 50.2 & 22.3 \\
\hline
DCZA error $\%$ & 49.8 & 49.8 & 49.8 & 22.7 \\
\hline
\end{tabular}
}
\caption{Error percentages of CoS and DCZA for learner 1, as a function of the number of learners present in the system.}
\vspace{-0.2in}
\label{tab:nlearn}
\end{table}

Fourthly, in Table \ref{tab:callcost}, we write the error percentages and the percentage of times each arm in ${\cal K}_i$ is chosen by learner $1$ as a function of the cost of calling other learners for CoS. We let $d_{1_j} = d$ for all $1_j \in {\cal M}_{-1}$. Since the goal of the learner is to maximize the expected accuracy minus cost, we see that learner 1 selects its own classification functions more often as the cost $d$ is increased. However, this results in higher error percentage since classification functions of learner 1 are suboptimal. 

\begin{table}[t]
\centering
{\fontsize{8}{6}\selectfont
\setlength{\tabcolsep}{.1em}
\begin{tabular}{|l|c|c|c|c|c|}
\hline
d & error $\%$ & training $\%$  & selection (except training/exploration) $\%$  \\
& & of learners 2,3,4 & of learners 1,2,3,4  \\ 
\hline
0 & 0.9 & 0.27, 0.23, 0.16 & 52.9, 47, 0.1, 0\\
\hline
0.5 & 1 & 0.27, 0.23, 0.16 & 53, 47, 0, 0 \\
\hline
0.7 & 23.7 &0.27, 0.23, 0.16 & 100, 0, 0, 0\\
\hline
\end{tabular}
}
\caption{Error and arm selection percentages as a function of calling cost}
\vspace{-0.25in}
\label{tab:callcost}
\end{table}

Finally, in Table \ref{tab:onlycontext}, we give the error percentage of learner 1 using CoS with parameter values Z1, when learner 1 only sends its context information to other learners but not its data. When called by learner 1, other learners do not predict based on their classification functions but they choose the prediction that has the highest percentage of being correct so far at the hypercube that the context of learner 1 belongs.
From these results we see that for S1 (two good classification functions), the error percentage of learner 1 is slightly higher than the error percentage when it sends also its data for contexts A1 and A3, while its error percentage is better for context A2. However, for S2 (no good classification functions), sending only context information produces very high error rates for all types of contexts. This suggests correlation of the context with the label and data is not enough to have low regret when only context information is sent. There should be classification functions which have low error rates.
Similar results hold for DCZA as well when only context information is sent between the learners. 

\begin{table}[t]
\centering
{\fontsize{8}{6}\selectfont
\setlength{\tabcolsep}{.1em}
\begin{tabular}{|l|c|c|c|c|}
\hline
(Setting) Algorithm & previous label (A1) & srcbytes (A2) & time   \\
& is context & is context & is context \\
\hline
(Z1,S1) CoS error $\%$ & 2.68 & 3.64 & 6.43  \\
\hline
(Z1,S2) CoS error $\%$ & 23.8 & 42.6 & 29 \\
\hline
\end{tabular}
}
\caption{Error percentages of CoS for learner 1, when learner 1 only sends its context information to the other learners.}
\vspace{-0.3in}
\label{tab:onlycontext}
\end{table}

\vspace{-0.1in}
\section{Conclusion} \label{sec:conc}
In this paper we considered two novel online learning algorithms for decentralized Big Data mining using context information about the high dimensional data. We provided several extensions of these algorithms to deal with challenges specific to Big Data mining such as concept drift, delayed feedback and ensemble learning. For some of these extensions, we proved sublinear regret results. We provided extensive numerical results both using real-world and synthetic data sets to illustrate how these algorithms operate under different data/context streams.
\rmv{An interesting research direction is to see the performance of CoS when combined with ensemble learning approaches. This will increase the communication and computation costs of the learners but the improvement in classification accuracy can be large compared to existing online ensemble learning methods such as SWA.}
\add{\vspace{-0.2in}}

\remove{
 \appendices
 \section{A bound on divergent series} \label{app:seriesbound}
 For $p>0$, $p \neq 1$,
\begin{align*}
\sum_{t=1}^{T} \frac{1}{t^p} \leq 1 + \frac{T^{1-p} -1}{1-p}
\end{align*}
\begin{proof}
See \cite{chlebus2009approximate}.
\end{proof}
}
\bibliographystyle{IEEE}
\bibliography{OSA}


\end{document}